%% file: uai_main.tex
\documentclass[accepted]{uai2026} % after acceptance, for a revised version; 
\usepackage{amsmath,algorithm, algorithmicx, algpseudocode,bbm}
\usepackage{textcomp}
\usepackage{url}
\usepackage{amsfonts,xcolor} 
\usepackage{amsthm}
\usepackage{amssymb}
\usepackage{lipsum}
\usepackage{amsmath}
\usepackage{ulem}
\usepackage{xcolor}
\usepackage{bbm}
\usepackage{graphicx}
\usepackage{booktabs}
\newtheorem{theorem}{Theorem}
\newtheorem{lemma}{Lemma}
\newtheorem{assumption}{Assumption}
\newtheorem{proposition}{Proposition}
\newtheorem{remark}{Remark}
\usepackage{multirow}
\usepackage{subcaption}
\usepackage{hyperref}

\newcommand{\rachid}[1]{{\color{blue}#1}}

\newcommand{\haoran}[1]{{\color{brown}#1}}

\definecolor{red}{rgb}{0,0,0}
\definecolor{orange}{rgb}{0,0,0}
\definecolor{brown}{rgb}{0,0,0}
\definecolor{blue}{rgb}{0,0,0}

\usepackage[american]{babel}
\usepackage{natbib} % has a nice set of citation styles and commands
\usepackage{mathtools} % amsmath with fixes and additions
\usepackage{booktabs} % commands to create good-looking tables
\usepackage{tikz} % nice language for creating drawings and diagrams

\title{FedSteer: Taming Extreme Gradient Staleness in Federated Learning with Corrective Projections and Caching}

\author[1]{\href{mailto:<haoranz@austin.utexas.edu>?Subject=FedSteer}{Haoran Zhang}{}\thanks{The work was performed while the author was affiliated with Carnegie Mellon University.}}
\author[3]{Cainã Figueiredo Pereira}
\author[4]{Marie Siew}
\author[5]{Xutong Liu}
\author[2]{Carlee Joe-Wong}
\author[3]{Rachid El-Azouzi}
\affil[1]{%
    Electrical and Computer Engineering Dept.\\
    The University of Texas at Austin\\
    Austin, Texas, USA
}
\affil[2]{%
    Electrical and Computer Engineering Dept.\\
    Carnegie Mellon University\\
    Pittsburgh, Pennsylvania, USA
}
\affil[3]{%
    CERI/LIA\\
    University of Avignon\\
    Avignon, France
  }
\affil[4]{%
    Information Systems Technology and Design Pillar\\
    Singapore University of Technology and Design\\
    Singapore
  }
\affil[5]{%
    Computer Science \& Systems Dept.\\
    University of Washington\\
    Tacoma, Washington, USA
  }
  
\begin{document}
\maketitle

\begin{abstract}
Federated learning (FL) is often subject to aggregation variance if clients do not consistently participate in training rounds.
While reusing stale model updates from inactive clients is a common technique to reduce this variance, we find that with skewed client participation, the resulting update staleness can become severe enough to destabilize training. 
To remedy this, we propose \texttt{FedSteer}, a novel method that constructs a gradient subspace from a cache of recent client gradients 
to serve as a low-dimensional representation of the current optimization landscape. 
FedSteer projects an active client's true gradient onto this subspace to find a set of optimal coordinates. For an inactive client, FedSteer reuses these coordinates with the now-evolved subspace drifted by other active clients.
This process effectively ``steers'' outdated gradients toward the current global objective. 
This is complemented by a selective caching strategy that identifies a representative client subset to form the subspace, reducing server memory. 
Experiments demonstrate that FedSteer significantly outperforms baselines, preventing performance collapse in challenging scenarios while delivering accuracy gains of over 7\% in others. Code is available \href{https://github.com/haoran-zh/FedSteer}{here}.
\end{abstract}

\input{sections/intro}
\input{sections/method}

\input{sections/experiment}

% \begin{contributions} % will be removed in pdf for initial submission 
% 					  % (without ‘accepted’ option in \documentclass)
%                       % so you can already fill it to test with the
%                       % ‘accepted’ class option
%     Briefly list author contributions. 
%     This is a nice way of making clear who did what and to give proper credit.
%     This section is optional.

%     H.~Q.~Bovik conceived the idea and wrote the paper.
%     Coauthor One created the code.
%     Coauthor Two created the figures.
% \end{contributions}

% \begin{acknowledgements} % will be removed in pdf for initial submission,
% 						 % (without ‘accepted’ option in \documentclass)
%                          % so you can already fill it to test with the
%                          % ‘accepted’ class option
%     Briefly acknowledge people and organizations here.

%     \emph{All} acknowledgements go in this section.
% \end{acknowledgements}

% References
\bibliography{references}

\newpage

\onecolumn

\title{FedSteer: Taming Extreme Gradient Staleness in Federated Learning with Corrective Projections and Caching
\\(Appendix)}
\maketitle

\appendix
\section{PROOFS}\label{app:proof}
% restart the theorem order
\setcounter{theorem}{1}
\setcounter{assumption}{0}
\begin{assumption}[Non-trivial gradients]
\label{assumption:scale}
For any cached gradient $\mathbf{h}_i^t, i\in\mathcal{X}$, its magnitude is strictly positive, i.e., $\|\mathbf{h}_i^t\|_2>0$. 
\end{assumption}

\begin{assumption}[$L$-smoothness]\label{assumption:L}
Each $F_{i}$ is L-smooth, and thus $F=\sum_{i\in \mathcal{N}} d_{i} F_{i}$ is also L-smooth.  
\end{assumption}

\begin{assumption}[Bounded variance at client-level]\label{assumption:B}
The stochastic gradient at each client is an unbiased estimator of the local gradient: $\mathbb{E}_{\xi_i\sim \mathcal{D}_i}[\nabla F_i(\mathbf{w},\xi_i)]=\nabla F_i(\mathbf{w})$, and its variance is bounded: $Var_{\xi_i\sim\mathcal{D}_i}(F_i(\mathbf{w},\xi_i))\leq 0$.
\end{assumption}

\begin{assumption}[Bounded variance across clients]
\label{assum:bounded_variance}
There exists a constant $\sigma_g^2 > 0$ such that the difference between the local gradient at the $i$-th client and the global gradient is bounded, that is
$$ \Vert \nabla F_i(\mathbf{w}) - \nabla F(\mathbf{w}) \Vert^2 \le \sigma_g^2, \quad \forall \mathbf{w}, i. $$
\end{assumption}

\begin{assumption}[Partial and heterogeneous client participation]
\label{assum:participation}
In each round $t$, client $i$ participates with a probability $p_i$, independently of previous rounds and other clients. $p_i$ is bounded: $p_{min}<p_i<p_{max}$. 
\end{assumption}

\begin{theorem}[Convergence analysis]
Under Assumptions \ref{assumption:L}-\ref{assum:participation}, $\eta_c\leq \frac{1}{4\sqrt{2E(E-1)}}$, and $\eta_s\leq \frac{1}{2L}$, the iterates $\{\mathbf{w}^t\}$ generated by FedSteer satisfy: 
\begin{align*}
&\min_{t\in[1,T]} \mathbb{E}\|\nabla F(\mathbf{w}^t)\|^2
\leq
\underbrace{\frac{4(F(\mathbf{w}^1)-F(\mathbf{w}^*))}{\eta_s T}}_{\text{iterate initialization error}} 
+ 
\underbrace{4 \eta_s L \frac{\sum_{t=1}^T \mathbb{E}\|\Delta^t-\mathbb{E}[\Delta^t]\|^2}{T}}_{\text{partial update variance error}} \\
&+\underbrace{\Gamma\sigma^2[2\eta_c^2L^2(E-1)+\frac{1}{TE}]  }_{\text{stochastic gradient error}}+\underbrace{(2+\Gamma) 8\eta_c^2L^2E(E-1) \sigma_g^2}_{\text{error from data heterogeneity}} 
+\underbrace{\frac{1}{T} \frac{\Gamma(1-p_{min})}{p_{avg}} \frac{1}{N} \sum_{i=1}^N \|\nabla F_i(\mathbf{w}^1)-\mathbf{h}_i^1 \|^2}_{\text{memory initialization error}}
\end{align*} 
where $\Gamma=\frac{p_{min} p_{avg}}{\eta_s L(2-p_{min})}$. 
\label{them:converge}
\end{theorem}

To prove Theorem 1, we first formalize the sources of randomness inherent in the training process.
\textbf{Source of Randomness: }
The system's stochasticity arises from two primary sources in each global round $t$. 
\begin{itemize}
    \item \textbf{Client sampling:} A subset of clients, denoted by the random set $\mathcal{A}_t$ is selected to participate in the training. This selection is determined by the sampling distribution $\mathbf{p}=\{p_i\}_{i\in\mathcal{N}}$.
    \item \textbf{Data sampling:} Each selected client $i\in\mathcal{A}_t$ computes stochastic gradients using mini-batches of data randomly sampled from its local dataset. 
\end{itemize}
To formalize our analysis, we define the random variables corresponding to these processes. Let $\xi_{i}^{(t,k)}$ denote the mini-batch of data points sampled by client $i$ at its $k$-th local step during round $t$.

To track the evolution of randomness across multiple rounds, we use the following notation:
\begin{itemize}
    \item $\mathbbm{1}_{i\in\mathcal{A}_t}$ is the indicator of client participation. 
    \item $\mathcal{A}^{(s:q)}=\{\mathcal{A}_{s}, \mathcal{A}_{s+1}, \dots, \mathcal{A}_{q}\}$ represents the sequence of client sets sampled from round $s$ to round $q$. 
    % Notice that for some client sampling methods, the sampling distribution in each round is also decided by the local training results. Therefore, 
    \item $\xi_{i}^{t}=\{\xi_{i}^{t,k}\}_{k=0}^{K-1}$ denotes the collection of all data mini-batches sampled by client $i$ during its local training in round t.
    \item $\xi^{t}=\{\xi_{i}^{t}\}_{i\in\mathcal{A}_t}$ denotes the collection of all data mini-batches sampled by active client $i$ at round $t$.
    \item $\xi^{(s:q)}=\{\xi^{s},\dots,\xi^q\}$ denotes the collection of all data mini-batches sampled by active clients from round $s$ to round $q$.
\end{itemize}

The stochastic progression of the algorithm from the first round to the current round $t$ can be expressed as: 
\begin{align}
\mathcal{H}^t=\{\mathcal{A}_1,\mathcal{A}_2,\dots, \mathcal{A}_{t-1},\xi^1, \xi^2, \dots, \xi^{t-1}\}.
\end{align}

We define the following pseudo-gradients for the proof:
\begin{align}
\text{Local stochastic pseudo-gradient: } &
g_i^t=\frac{1}{E} \sum_{e=1}^{E-1} \nabla F_i(w_i^{t,e},\xi_i^{t,e}),
\\
\text{Local pseudo-gradient: }&
\overline{g}_i^t=\frac{1}{E} \sum_{e=1}^{E-1} \nabla F_i(w_i^{t,e}),
\\
\text{Global stochastic pseudo-gradient: }&
g^t=\sum_{i\in\mathcal{N}} d_i g_i^t,\\
\text{Global pseudo-gradient: }&
\overline{g}^t=\sum_{i\in\mathcal{N}} d_i \overline{g}_i^t,
\\
\text{Global stale pseudo-gradient: }&
h^t=\sum_{i\in\mathcal{N}} d_i h_i^t,
\\
\text{Global estimate pseudo-gradient: }&
\hat{g}^t=\sum_{i\in\mathcal{N}} d_i \hat{g}_i^t,
\end{align}
where $E$ is the number of local epochs, following the same definition of the paper. In the proof, we also use $K$ to denote the number of local epochs, in the case where $E$ may be misunderstood as the compute of expectation. 
% Starting from supporting Lemmas (B.2 in FedStale)
\begin{lemma}[Descent lemma]
Let $F : \mathbb{R}^d \to \mathbb{R}$ be an $L$-smooth function, optimized via the sequence of parameters $\{w^{t}\}$. At each iteration $t$, an SGD update is made according to a learning rate $\eta_s$ and a stochastic gradient $\Delta^{t}$. Let $\mathbb{E}_{\mathcal{A}_{t},\xi^{t}|\mathcal{H}^{t}}[\Delta^{t}] = \bar{g}^{t}$. Then, the expected reduction in $F$ after one iteration is bounded by:
\begin{align}
\mathbb{E}_{\mathcal{A}_{t},\xi^{t}|\mathcal{H}^{t}}[F(w^{t+1})] &\leq F(w^{t}) - \frac{\eta_s}{2}\left[\|\nabla F(w^{t})\|^2 + \|\bar{g}^{t}\|^2 \right. \nonumber\\
&\quad \left.- \|\bar{g}^{t} - \nabla F(w^{t})\|^2\right] + \frac{\eta_s^2 L}{2} \mathbb{E}_{\mathcal{A}_{t},\xi^{t}|\mathcal{H}^{t}}\|\Delta^{t}\|^2.
\end{align}
\end{lemma}

\begin{proof}
By the $L$-smoothness of $F$, it follows that:
\begin{align}
F(w^{t+1}) &\leq F(w^{t}) + \langle \nabla F(w^{t}), w^{t+1} - w^{t} \rangle + \frac{L}{2} \|w^{t+1} - w^{t}\|^2 \nonumber\\
&\leq F(w^{t}) - \eta_s \langle \nabla F(w^{t}), \Delta^{t} \rangle + \frac{\eta_s^2 L}{2} \|\Delta^{t}\|^2,\label{eq:lemma1}
\end{align}
where Eq. \eqref{eq:lemma1} applies the update rule $w^{t+1} = w^{t} - \eta_s \Delta^{t}$.

Taking the expectation over the randomness at the $t$-th round, due to client participation (inherent in $\xi^{t}$) and stochastic gradients (inherent in $\xi^{t} := \{\xi^{(t,k)}_i\}_{i,k}$), yields:
\begin{align}
\mathbb{E}_{\mathcal{A}_{t},\xi^{t}|\mathcal{H}^{t}}[F(w^{t+1})] &\leq F(w^{t}) - \eta_s \mathbb{E}_{\mathcal{A}_{t},\xi^{t}|\mathcal{H}^{t}}\langle \nabla F(w^{t}), \Delta^{t} \rangle + \frac{\eta_s^2 L}{2} \mathbb{E}_{\mathcal{A}_{t},\xi^{t}|\mathcal{H}^{t}}\|\Delta^{t}\|^2 \nonumber\\
&\leq F(w^{t}) - \eta_s \left[\langle \nabla F(w^{t}), \bar{g}^{t} \rangle\right] + \frac{\eta_s^2 L}{2} \mathbb{E}_{\mathcal{A}_{t},\xi^{t}|\mathcal{H}^{t}}\|\Delta^{t}\|^2 \nonumber\\
&\leq F(w^{t}) - \frac{\eta_s}{2} \left[\|\nabla F(w^{t})\|^2 + \|\bar{g}^{t}\|^2 \right. \nonumber\\
&\quad \left.- \|\bar{g}^{t} - \nabla F(w^{t})\|^2\right] + \frac{\eta_s^2 L}{2} \mathbb{E}_{\mathcal{A}_{t},\xi^{t}|\mathcal{H}^{t}}\|\Delta^{t}\|^2,
\end{align}
where the second inequality uses $\mathbb{E}_{\mathcal{A}_{t},\xi^{t}|\mathcal{H}^{t}}[\Delta^{t}] = \bar{g}^{t}$ and the third inequality applies the identity $\|a - b\|^2 = \|a\|^2 + \|b\|^2 - 2\langle a, b \rangle$.
\end{proof}

\begin{lemma}[Expected value of the local stochastic pseudo-gradients]\label{lemma:2}
If the stochastic gradients are unbiased (Assumption 3), the following identity holds:
\begin{align}
    \mathbb{E}_{\xi_i^{t}|\mathcal{H}^{t}} \left[ g_i^{t} \right] = \bar{g}_i^{t}.
\end{align}
\end{lemma}

\begin{proof}
% We observe that the randomness in the iterate for a specific client, $w_i^{(t,k)}$, is influenced both by the sequence of events up to time $t$ (denoted as $H^{t}$) and by the random batches used for training up to the $k$-th iteration ($\xi_i^{(t,0:k-1)}$).

We decompose the expected value of the gradient $\nabla F_i(w_i^{(t,k)}, \xi_i^{(t,k)})$ as:
\begin{align}
    \mathbb{E}_{\xi_i^{(t,0:k)} |\mathcal{H}^{t}} \left[ \nabla F_i(w_i^{(t,k)}, \xi_i^{(t,k)}) \right] 
    = \mathbb{E}_{\xi_i^{(t,0:k-1)} |\mathcal{H}^{t}} \left[ \mathbb{E}_{\xi_i^{(t,k)} | \xi_i^{(t,0:k-1)}, H^{t}} \left[ \nabla F_i(w_i^{(t,k)}, \xi_i^{(t,k)}) \right] \right].\label{eq:lemma2-35}
\end{align}

We finally use Assumption 3 to conclude that $\mathbb{E}_{\xi_i^{(t,k)} | \xi_i^{(t,0:k-1)}, \mathcal{H}^{t}} \left[ \nabla F_i(w_i^{(t,k)}, \xi_i^{(t,k)}) \right] = \nabla F_i(w_i^{(t,k)})$.

Below, we present the detailed derivations of the proof.

\begin{align}
    \mathbb{E}_{\xi_i^{t}|\mathcal{H}^{t}} \left[ g_i^{t} \right] &= \frac{1}{K} \sum_{k=0}^{K-1} \mathbb{E}_{\xi_i^{t}|\mathcal{H}^{t}} \left[ \nabla F_i(w_i^{(t,k)}, \xi_i^{(t,k)}) \right]\label{eq:lemma2-36} \\
    &= \frac{1}{K} \mathbb{E}_{\xi_i^{(t,0)} |\mathcal{H}^{t}} \left[ \nabla F_i(w^{t}, \xi_i^{(t,0)}) \right] + \frac{1}{K} \mathbb{E}_{\xi_i^{(t,0)}, \xi_i^{(t,1)} |\mathcal{H}^{t}} \left[ \nabla F_i(w_i^{(t,1)}, \xi_i^{(t,1)}) \right] + \cdots \\
    &\quad + \frac{1}{K} \mathbb{E}_{\xi_i^{(t,0:K-1)} |\mathcal{H}^{t}} \left[ \nabla F_i(w_i^{(t,K-1)}, \xi_i^{(t,K-1)}) \right] \label{eq:lemma2-37}\\
    &= \frac{1}{K} \nabla F_i(w^{t}) + \frac{1}{K} \mathbb{E}_{\xi_i^{(t,0)} |\mathcal{H}^{t}} \left[ \mathbb{E}_{\xi_i^{(t,1)} | \xi_i^{(t,0)}, \mathcal{H}^{t}} \left[ \nabla F_i(w_i^{(t,1)}, \xi_i^{(t,1)}) \right] \right] + \cdots \\
    &\quad + \frac{1}{K} \mathbb{E}_{\xi_i^{(t,0:K-2)} |\mathcal{H}^{t}} \left[ \mathbb{E}_{\xi_i^{(t,K-1)} | \xi_i^{(t,0:K-2)}, \mathcal{H}^{t}} \left[ \nabla F_i(w_i^{(t,K-1)}, \xi_i^{(t,K-1)}) \right] \right] \label{eq:lemma2-38}\\
    &= \frac{1}{K} \left[ \nabla F_i(w^{t}) + \mathbb{E}_{\xi_i^{(t,0)} |\mathcal{H}^{t}} \left[ \nabla F_i(w_i^{(t,1)}) \right] + \cdots + \mathbb{E}_{\xi_i^{(t,0:K-2)} |\mathcal{H}^{t}} \left[ \nabla F_i(w_i^{(t,K-1)}) \right] \right]\label{eq:lemma2-39}\\
    &= \frac{1}{K} \sum_{k=0}^{K-1} \nabla F_i(w_i^{(t,k)}) = \bar{g}_i^{t},\label{eq:lemma2-40}
\end{align}
where Eq. \eqref{eq:lemma2-36} uses the definition of $g_i^{t}$, Eq. \eqref{eq:lemma2-37} makes explicit the dependency of the iterate $w_i^{(t,k)}$ on the random batches $\xi_i^{(t,0:k-1)}$, Eq. \eqref{eq:lemma2-38} uses the law of total expectation given in \eqref{eq:lemma2-35}, Eq. \eqref{eq:lemma2-39} applies the unbiasedness of the stochastic gradient (Assumption 3), and Eq. \eqref{eq:lemma2-40} uses the definition of $\bar{g}_i^{t}$.
\end{proof}

\begin{lemma}[Variance of the local stochastic pseudo-gradients]\label{lemma:3}
If the variance of the local stochastic gradients is bounded by $\sigma^2$ (Assumption 3), the following inequality holds:
\begin{align}
    \mathbb{E}_{\xi_i^{t}|\mathcal{H}^{t}} \left\| g_i^{t} - \bar{g}_i^{t} \right\|^2 \leq \frac{\sigma^2}{K}.
\end{align}
\end{lemma}

\begin{proof}
The proof builds on similar observations to those presented in Lemma \ref{lemma:2}, but additionally relies on the bounded variance of local stochastic gradients (Assumption 3).

Below, the detailed derivations.

\begin{align}
    &\mathbb{E}_{\xi_i^{t}|\mathcal{H}^{t}}\|g_i^{t} - \bar{g}_i^{t}\|^2  \\
    &= \mathbb{E}_{\xi_i^{t}|\mathcal{H}^{t}} \left\|\frac{1}{K}\sum_{k=0}^{K-1} \left[ \nabla F_i(\mathbf{w}_t^{(t,k)}, \xi_t^{(k)}) - \nabla F_i(\mathbf{w}_t) \right]^2 \right\| \label{eq:lemma3-42} \\
    &= \frac{1}{K^2} \sum_{k=0}^{K-1} \mathbb{E}_{\xi_i^{t}|\mathcal{H}^{t}}  \left\| \nabla F_i(\mathbf{w}_i^{(t,k)}, \xi_i^{(t,k)}) - \nabla F_i(\mathbf{w}_i^{(t,k)})\right\|^2  \\
    &\quad + \frac{1}{K^2} \sum_{k=0}^{K-1} \sum_{\substack{k' = 0 \\ k' \neq k}}^{K-1} \mathbb{E}_{\xi_i^{t}|\mathcal{H}^{t}} \left\langle \nabla F_i(w_i^{(t,k)}, \xi_i^{(t,k)}) - \nabla F_i(w_i^{(t,k)}), \nabla F_i(w_i^{(t,k')}, \xi_i^{(t,k')}) - \nabla F_i(w_i^{(t,k')}) \right\rangle,\label{eq:lemma3-43}
\end{align}

where Eq. \eqref{eq:lemma3-42} applies the definitions for $\mathbf{g}^{t}$ and $\mathbf{g}^{(t,0)}$, and Eq. \eqref{eq:lemma3-43} expands the squared norm. To show that the second term in \eqref{eq:lemma3-43} is zero, we use the law of total expectation in a similar way as in \eqref{eq:lemma2-35}. Indeed, denote $k'' = \max\{k,k'\}$. The following relation holds:

\begin{align}
&\mathbb{E}_{\xi_i^{(t, 0:k'')}|\mathcal{H}^{t}}\left[\nabla F_i\left(\mathbf{w}_i^{(t, k'')}, \xi_i^{(t, k'')}\right) - \nabla F_i\left(\mathbf{w}_i^{(t, k'')}\right)\right] \\
    &= \mathbb{E}_{\xi_i^{(t, 0:k''-1)}|\mathcal{H}^{t}}\left[ \underbrace{\mathbb{E}_{\xi_i^{(t, k'')}|\xi_{i}^{(t,0:k''-1)},\mathcal{H}^{t}}\left[\nabla F_i\left(\mathbf{w}_i^{(t, k'')}, \xi_i^{(t, k'')}\right) - \nabla F_i\left(\mathbf{w}_i^{(t, k'')}\right)\right]}_{\text{=0 by Assumption 3}}\right] \\
    &\quad = 0
\end{align}

Therefore, only the first term remains:

\begin{align}
&\mathbb{E}_{\xi_i^{t}|\mathcal{H}^{t}}\|g_i^{t} - \bar{g}_i^{t}\|^2 \\
&= \frac{1}{K^2} \sum_{k=0}^{K-1} \mathbb{E}_{\xi_i^{t}|\mathcal{H}^{t}}\Bigg\|\nabla F_i\left(\mathbf{w}_i^{(t,k)}, \xi_i^{(t,k)}\right) - \nabla F_i\left(\mathbf{w}_i^{(t,k)}\right)\Bigg\|^2 \\
&= \frac{1}{K^2} \bigg[\mathbb{E}_{\xi_i^{(t,0)}}\|\nabla F_i(w^{t},\xi_i^{(t,0)})-\nabla F_i(w^{t})\|^2+\cdots\\
&\quad \mathbb{E}_{\xi_i^{(t,0:K-1)}|\mathcal{H}^{t}} \|\nabla F_i(w_i^{(t,K-1)}, \xi_i^{(t,K-1)})-\nabla F_i(w_i^{(t,K-1)})\|^2 \bigg]\\
&\leq \frac{1}{K^2} \sum_{k=1}^{K-1} \sigma^2 = \frac{\sigma^2}{K}
\end{align}
\end{proof}

\begin{lemma}[Variance of the global stochastic pseudo-gradient] Assuming that client participation outcomes $\xi_i^{t}$ are decided by $\{p_i^t\}$, and that the variance of the local stochastic gradients is bounded by $\sigma^2$ (Assumption 3), the following inequality holds:
\begin{align}
\mathbb{E}_{\mathcal{A}_t,\xi^{t}|\mathcal{H}^{t}} \left\| \frac{1}{N} \sum_{i=1}^{N} \frac{\xi_i^{t}}{p_i^t} \left(g_i^{t} - \bar{g}_i^{t}\right) \right\|^2 \leq \left(\frac{1}{N} \sum_{i=1}^{N} \frac{1}{p_i^t} \right) \frac{\sigma^2}{NK}.
\end{align}
\end{lemma}

\begin{proof}
The proof starts by expanding the squared norm of the average stochastic gradient deviations into a variance term accounting for individual client gradients and a covariance term between gradients from different clients:

\begin{align}
&\mathbb{E}_{\mathcal{A}_t, \xi^{t}|\mathcal{H}^{t}} \left\| \frac{1}{N} \sum_{i=1}^{N} \frac{\xi_i^{t}}{p_i^t} \left(g_i^{t} - \bar{g}_i^{t}\right) \right\|^2 \\
&= \mathbb{E}_{\mathcal{A}_t, \xi^{t}|\mathcal{H}^{t}} \bigg[ \frac{1}{N^2} \sum_{i=1}^{N} \frac{\left(\xi_i^{t}\right)^2}{(p_i^t)^2} \left\| g_i^{t} - \bar{g}_i^{t} \right\|^2 + \frac{1}{N^2} \sum_{i=1}^{N} \sum_{\substack{i'=1 \\ i' \neq i}}^{N} \frac{\xi_i^{t} \xi_{i'}^{t}}{p_i^t p_{i'}^t} \left\langle g_i^{t} - \bar{g}_i^{t}, g_{i'}^{t} - \bar{g}_{i'}^{t} \right\rangle \Bigg]
\end{align}

We leverage the linearity of expectation, the independence of client participation ($\xi^{t}$) and batch sampling among clients ($\xi_i^{t}$ and $\xi_{i'}^{t}$), and Lemma \ref{lemma:2} to show that:

\begin{align}
&\mathbb{E}_{\mathcal{A}_t, \xi^{t}|\mathcal{H}^{t}} \left[ \frac{\xi_i^{t} \xi_{i'}^{t}}{p_i^t p^t_{i'}} \left\langle g_i^{t} - \bar{g}_i^{t}, g_{i'}^{t} - \bar{g}_{i'}^{t} \right\rangle \right] \\
&= \frac{\mathbb{E}_{\xi^{t}|\mathcal{H}^{t}}[\xi_i^{t} \xi_{i'}^{t}]}{p_i^t p_{i'}^t} \mathbb{E}_{\xi^{t}|\mathcal{H}^{t}} \left[ \left\langle g_i^{t} - \bar{g}_i^{t}, g_{i'}^{t} - \bar{g}_{i'}^{t} \right\rangle \right] \\
&= \frac{\mathbb{E}_{\xi^{t}|\mathcal{H}^{t}} \left[\xi_i^{t} \xi_{i'}^{t}\right]}{p_i^t p_{i'}^t} \left\langle \mathbb{E}_{\xi_i^{t}|\mathcal{H}^{t}} \left[g_i^{t} - \bar{g}_i^{t}\right], \mathbb{E}_{\xi_{i'}^{t}|\mathcal{H}^{t}} \left[g_{i'}^{t} - \bar{g}_{i'}^{t}\right] \right\rangle = 0,
\end{align}

Finally, we bound the remaining term using Lemma \ref{lemma:3}:

\begin{align}
\mathbb{E}_{\mathcal{A}_t, \xi^{t}|\mathcal{H}^{t}} \left\| \frac{1}{N} \sum_{i=1}^{N} \frac{\xi_i^{t}}{p_i} \left(g_i^{t} - \bar{g}_i^{t}\right) \right\|^2
&= \frac{1}{N^2} \sum_{i=1}^{N} \frac{\mathbb{E}_{\xi_i^{t}|\mathcal{H}^{t}} \left[\left(\xi_i^{t}\right)^2\right]}{(p_i^t)^2} \mathbb{E}_{\xi_i^{t}|\mathcal{H}^{t}} \left\| g_i^{t} - \bar{g}_i^{t} \right\|^2 \\
&\leq \left(\frac{1}{N} \sum_{i=1}^{N} \frac{1}{p_i^t}\right) \frac{\sigma^2}{NK}
\end{align}
\end{proof}

\begin{lemma}[Client drift due to multiple local iterations]\label{lemma:5}
Under bounded local stochastic gradient variance $(\sigma^2$, as per Assumption 3) and the client learning rate $\eta_c \leq \frac{1}{2LK}$, the expected squared deviation of a client’s pseudo-gradient $(\bar{g}_i^{t})$ from its local gradient $(\nabla F_i(w^{t}))$ is bounded as:
\begin{align}
\mathbb{E}_{\xi_i^{t}|\mathcal{H}^{t}} \left\|\bar{g}_i^{t} - \nabla F_i(w^{t})\right\|^2 \leq 2\eta_c^2 L^2 K(K-1) \left[\frac{\sigma^2}{K} + 2 \left\|\nabla F_i(w^{t})\right\|^2 \right]\label{lemma5-55}
\end{align}

Additionally, if the variance of local gradients is uniformly bounded across clients (by $\sigma_g^2$, as per Assumption 4):
\begin{align}
\mathbb{E}_{\xi_i^{t}|\mathcal{H}^{t}} \left\|\bar{g}_i^{t} - \nabla F_i(w^{t})\right\|^2 \leq 2\eta_c^2 L^2 K(K-1) \left[\frac{\sigma^2}{K} + 4\sigma_g^2 + 4 \left\|\nabla F(w^{t})\right\|^2 \right]. \label{lemma5-56}
\end{align}
\end{lemma}
The bound in Eq. (56) captures that, when the number of local iterations $K$ equals 1, $\bar{g}_i^{t}$ and $\nabla F_i(\mathbf{w}^{t})$ become equivalent.

\begin{proof} 
\begin{align}
\mathbb{E}_{\xi_i^{t}|\mathcal{H}^{t}} \left\|\bar{g}_i^{t} - \nabla F_i(w^{t})\right\|^2 &= \mathbb{E}_{\xi_i^{t}|\mathcal{H}^{t}} \left\|\frac{1}{K} \sum_{k=0}^{K-1} \left(\nabla F_i(w_i^{(t,k)}) - \nabla F_i(w^{t})\right)\right\|^2 \\
&\leq \frac{1}{K} \sum_{k=0}^{K-1} \mathbb{E}_{\xi_i^{t}|\mathcal{H}^{t}} \left\|\nabla F_i(w_i^{(t,k)}) - \nabla F_i(w^{t})\right\|^2 \\
&\leq \frac{L^2}{K} \sum_{k=0}^{K-1} \mathbb{E}_{\xi_i^{t}|\mathcal{H}^{t}} \left\|w_i^{(t,k)} - w^{t}\right\|^2 \label{eq:lemma5-59}
\end{align}
The above inequalities use Jensen's inequality and L-smoothness assumption. Next, the individual difference is bounded as: 
\begin{align}
&\mathbb{E}_{\xi_i^{t}|\mathcal{H}^{t}}\|w_i^{(t,k)}-w^{t}\|^2\\
&=\eta_c^2 \mathbb{E}_{\xi_i^{t}|\mathcal{H}^{t}} \|\sum_{k'=0}^{k-1} \nabla F_i(w_i^{(t,k')},\xi_i^{(t,k')})\|^2\label{eq:lemma5-60}\\
&=\eta_c^2 \bigg[\mathbb{E}_{\xi_i^{t}|\mathcal{H}^{t}} \|\sum_{k'=0}^{k-1} [\nabla F_i(w_i^{(t,k')},\xi_i^{(t,k')})-\nabla F_i(w_i^{(t,k')})]\|^2+\mathbb{E}_{\xi_i^{t}|\mathcal{H}^{t}} \|\sum_{k'=0}^{k-1} \nabla F_i(w_i^{(t,k')})\|^2\bigg]\label{eq:lemma5-61}\\
&\leq \eta_c^2 \left[ \sum_{k'=0}^{k-1} \mathbb{E}_{\xi_i^{t}|\mathcal{H}^{t}} \left\|\nabla F_i\left(w_i^{(t,k')}, \xi_i^{(t,k')}\right) - \nabla F_i\left(w_i^{(t,k')}\right)\right\|^2 + k \sum_{k'=0}^{k-1} \mathbb{E}_{\xi_i^{t}|\mathcal{H}^{t}} \left\|\nabla F_i\left(w_i^{(t,k')}\right)\right\|^2 \right]\label{eq:lemma5-62}\\
&\leq \eta_c^2 \left[ k \sigma^2 + k \sum_{k'=0}^{k-1} \mathbb{E}_{\xi_i^{t}|\mathcal{H}^{t}} \left\|\nabla F_i\left(w_i^{(t,k')}\right) - \nabla F_i\left(w^{t}\right) + \nabla F_i\left(w^{t}\right)\right\|^2 \right]\label{eq:lemma5-63}\\
&\leq \eta_c^2 \left[ k \sigma^2 + 2k \sum_{k'=0}^{k-1} \left[ L^2 \mathbb{E}_{\xi_i^{t}|\mathcal{H}^{t}} \left\|w_i^{(t,k')} - w^{t}\right\|^2 + \left\|\nabla F_i(w^{t})\right\|^2 \right] \right]\label{eq:lemma5-64}
\end{align}
where Eq. \eqref{eq:lemma5-60} applies the local update rule, Eq. \eqref{eq:lemma5-61} leverages the local stochastic gradient unbiasedness (as per Lemma \ref{lemma:2}) and its bias-variance decomposition. Eq. \eqref{eq:lemma5-62} involves squaring the former term, zeroing the cross terms, and applying Jensen's inequality to the latter term. Eq. \eqref{eq:lemma5-63} accounts for the bounded variance of local stochastic gradients in the former term (Lemma \ref{lemma:3}), and modifies the latter term by adding and subtracting the initial local gradient ($\nabla F_i(w^{t})$); finally Eq. \eqref{eq:lemma5-64} uses the norm inequality and the L-smoothness of local objectives. 

Summing over $k=0,\dots, K-1$, it yields: 
\begin{align}
&\frac{1}{K} \sum_{k=0}^{K-1} \mathbb{E}_{\xi_i^{t}|\mathcal{H}^{t}} \left\|w_i^{(t,k)} - w^{t}\right\|^2\\
&\leq \frac{\eta_c^2 \sigma^2}{K} \sum_{k=0}^{K-1} k + \frac{2 \eta_c^2 L^2}{K} \sum_{k=0}^{K-1} k \sum_{k'=0}^{k-1} \mathbb{E}_{\xi_i^{t}|\mathcal{H}^{t}} \left\|w_i^{(t,k')} - w^{t}\right\|^2 
+ \frac{2 \eta_c^2}{K} \sum_{k=0}^{K-1} k \sum_{k'=0}^{k-1} \left\|\nabla F_i(w^{t})\right\|^2\\
&\leq \eta_c^2 (K-1) \sigma^2 + 2 \eta_c^2 L^2 K(K-1) \left[\frac{1}{K} \sum_{k=0}^{K-1} \mathbb{E}_{\xi_i^{t}|\mathcal{H}^{t}} \left\|w_i^{(t,k)} - w^{t}\right\|^2 \right] 
+ 2 \eta_c^2 K(K-1) \left\|\nabla F_i(w^{t})\right\|^2,\label{eq:lemma5-66}
\end{align}
where Eq. \eqref{eq:lemma5-66} uses $\sum_{k'=0}^{k-1}\|w_i^{(t,k')}-w^{t}\|^2\leq \sum_{k=0}^{K-1} \|w_i^{(t,k)}-w^{t}\|^2$, and $\sum_{k=0}^{K-1} k = \frac{1}{2} (K-1)K$.

Define $D:=2\eta_c^2L^2K(K-1)$. Choose $\eta_c$ small enough such that $D\leq \frac{1}{2}$ so $\eta_c\leq \frac{1}{2LK}$. Rearranging the terms: 
\begin{align}
\frac{1}{K} \sum_{k=0}^{K-1} \mathbb{E}_{\xi_i^{t}|\mathcal{H}^{t}} \left\|w_i^{(t,k)} - w^{t}\right\|^2 \leq \frac{\eta_c^2 (K-1) \sigma^2}{1 - D} + \frac{2 \eta_c^2 K(K-1)}{1 - D} \left\|\nabla F_i(w^{t})\right\|^2 \label{eq:lemma5-67}
\end{align}
Substituting \eqref{eq:lemma5-67} back to \eqref{eq:lemma5-59}:
\begin{align}
\mathbb{E}_{\xi_i^{t}|\mathcal{H}^{t}} \left\|\bar{g}_i^{t} - \nabla F_i(w^{t})\right\|^2 
&\leq \frac{D}{2(1-D)} \frac{\sigma^2}{K} + \frac{D}{1-D} \left\|\nabla F_i(w^{t})\right\|^2 \\
&\leq D \frac{\sigma^2}{K} + 2D \left\|\nabla F_i(w^{t})\right\|^2,\label{eq:lemma5-69}
\end{align}
where Eq. \eqref{eq:lemma5-69} uses $D \leq \frac{1}{2}$. Replacing $D := 2\eta_c^2 L^2 K(K-1)$ into \eqref{eq:lemma5-69} completes the proof of Inequality \eqref{lemma5-55}. Additionally, inequality \eqref{lemma5-56} removes the dependency on $\nabla F_i(w^{t})$ by adding and subtracting $\nabla F(w^{t})$ in the squared norm:
\begin{align}
\mathbb{E}_{\xi_i^{t}|\mathcal{H}^{t}} \left\|\bar{g}_i^{t} - \nabla F_i(w^{t})\right\|^2 &\leq D \frac{\sigma^2}{K} + 2D \left\|\nabla F_i(w^{t}) - \nabla F(w^{t}) + \nabla F(w^{t})\right\|^2\\
&\leq D \frac{\sigma^2}{K} + 4D \left\|\nabla F_i(w^{t}) - \nabla F(w^{t})\right\|^2 + 4D \left\|\nabla F(w^{t})\right\|^2\\
&\leq D \frac{\sigma^2}{K} + 4D \sigma_g^2 + 4D \left\|\nabla F(w^{t})\right\|^2,\label{eq:lemma5-72}
\end{align}
Replacing $D:=2\eta_c^2 L^2 K(K-1)$ into \eqref{eq:lemma5-72} concludes the proof of inequality \eqref{lemma5-56}. 
\end{proof}
% \textcolor{red}{[We do not need lemma 6 (update variance bound) as FedStale, so we omit the lemma about variance bound. ]}
\begin{lemma}[Bound on the memory term]\label{lemma:memeory}
Define $H^{t}$ as the divergence between the local gradient and the historical pseudo-gradient at time $t$: 
\begin{align}
H^{t}=\frac{1}{N}\sum_{i=1}^N \|\nabla F_i(w^{t-1})-h_i^{t}\|^2
\end{align}
The expected historical error $H^{t+1}$ is recursively bounded as: 
\begin{align}
\mathbb{E}_{\mathcal{A}_{t}, \xi^{t} |\mathcal{H}^{t}} \left[ H^{t+1} \right] 
&\leq \left( \frac{1}{N} \sum_{i=1}^{N} p_i^t \right) \frac{\sigma^2}{K} + \frac{1}{N} \sum_{i=1}^{N} p_i^t \mathbb{E}_{\xi_i^{t} |\mathcal{H}^{t}} \left\| \bar{g}_i^{t} - \nabla F_i(w^{t}) \right\|^2\\
&+ \eta_s^2 L^2 \left( 1 + \frac{1}{C} \right) \left( 1 - \frac{1}{N} \sum_{i=1}^{N} p_i^t \right) \left\| \Delta^{t-1} \right\|^2 + (1+C)(1-p_{\text{min}}) H^{t}\label{eq:lemma7-89}
\end{align}
\end{lemma}
\begin{proof}
the proof starts by definition of $H^{t+1}$: 
\begin{align}
&\mathbb{E}_{\mathcal{A}_{t}, \xi^{t} |\mathcal{H}^{t}} \left[ H^{t+1} \right]\\
&=\frac{1}{N} \sum_{i=1}^{N} \mathbb{E}_{\xi_i^{t}|\mathcal{H}^{t}} \left[ \mathbb{E}_{\xi_i^{t} \mid \xi_i^{t}, \mathcal{H}^{t}} \left\| \nabla F_i(\mathbf{w}^{t}) - \mathbf{h}_i^{t+1} \right\|^2 \right]\label{eq:lemma7-90}\\
&= \frac{1}{N} \sum_{i=1}^{N} \left[ p_i^t\, \mathbb{E}_{\xi_i^{t} \mid \mathcal{H}^{t}} \left\| \nabla F_i(\mathbf{w}^{t}) - \mathbf{g}_i^{t} \right\|^2 + (1 - p_i^t) \left\| \nabla F_i(\mathbf{w}^{t}) - \mathbf{h}_i^{t} \right\|^2 \right]\label{eq:lemma7-91}\\
&\leq \frac{1}{N} \sum_{i=1}^{N} p_i^t\, \mathbb{E}_{\xi_i^{t} \mid \mathcal{H}^{t}} \left\| \mathbf{g}_i^{t} - \bar{\mathbf{g}}_i^{t} \right\|^2 + \frac{1}{N} \sum_{i=1}^{N} p_i^t\, \mathbb{E}_{\xi_i^{t} \mid \mathcal{H}^{t}} \left\| \bar{\mathbf{g}}_i^{t} - \nabla F_i(\mathbf{w}^{t}) \right\|^2\\
&+ \frac{\left(1 + \frac{1}{C}\right)}{N} \sum_{i=1}^{N} (1 - p_i^t) \left\| \nabla F_i(\mathbf{w}^{t}) - \nabla F_i(\mathbf{w}^{t-1}) \right\|^2 + \frac{\left(1 + C\right)}{N} \sum_{i=1}^{N} (1 - p_i^t) \left\| \nabla F_i(\mathbf{w}^{t-1}) - \mathbf{h}_i^{t} \right\|^2\label{eq:lemma7-92}\\
&\leq \left( \frac{1}{N} \sum_{i=1}^{N} p_i^t\right) \frac{\sigma^2}{K} + \frac{1}{N} \sum_{i=1}^{N} p_i^t\, \mathbb{E}_{\xi_i^{t} \mid \mathcal{H}^{t}} \left\| \bar{\mathbf{g}}_i^{t} - \nabla F_i(\mathbf{w}^{t}) \right\|^2\\
&+ \frac{\eta_s^2 L^2 \left(1 + \frac{1}{C}\right)}{N} \sum_{i=1}^{N} (1 - p_i^t) \left\| \Delta^{t-1} \right\|^2 + \frac{(1 + C)}{N} \sum_{i=1}^{N} (1 - p_i^t) \left\| \nabla F_i(\mathbf{w}^{t-1}) - \mathbf{h}_i^{t} \right\|^2,\label{eq:lemma7-93}
\end{align}
where Eq. \eqref{eq:lemma7-90} uses the law of total expectation to separate expectations on client participation ($\xi_i^{t}$) and batch sampling ($\xi_i^{t}$); Eq. \eqref{eq:lemma7-91} solves the inner expectation with respect to client participation; Eq. \eqref{eq:lemma7-92} manipulates the first term by adding and subtracting $\bar{g}_i^{t}$, then leverages the bounded variance of the local stochastic pseudo-gradients, and similarly corrects the second term with $\nabla F_i(w^{t-1})$, then applies the norm inequality $\|a+b\|^2\leq (1+\frac{1}{C})\|a\|^2+(1+C) \|b\|^2$ for any positive $C$; Eq. \eqref{eq:lemma7-93} is derived from the $L$-smoothness property of local objectives. 

The final expression in Eq. \eqref{eq:lemma7-89} is derived by observing that $\sum_{i=1}^N (1-a_i)b_i\leq (1-a_{min}) \sum_{i=1}^N b_i$
\end{proof}

% \textcolor{red}{[omit lemma 8, initial variance bound]}

\begin{lemma}[Bound on the memory term - Initial condition]
\begin{align}
\mathbb{E}_{\mathbbm{1}^{1}, \xi^{1}} \left[ H^{2} \right] &\leq \left( \frac{1}{N} \sum_{i=1}^{N} p_i^t\right) \frac{\sigma^2}{K} + \frac{1}{N} \sum_{i=1}^{N} p_i^t\, \mathbb{E}_{\xi_i^{1}} \left\| \nabla F_i(\mathbf{w}^{1}) - \bar{\mathbf{g}}_i^{1} \right\|^2 \\
&+ (1 - p_{\min}) \frac{1}{N} \sum_{i=1}^{N} \left\| \nabla F_i(\mathbf{w}^{1}) - \mathbf{h}_i^{1} \right\|^2
\end{align}
\end{lemma}
\begin{proof}
The proof starts with the definition of $H^{2}$: 
\begin{align}
&\mathbb{E}_{1^{1}, \xi^{1}} \left[ H^{2} \right] \\
&= \frac{1}{N} \sum_{i=1}^{N} \mathbb{E}_{\xi_i^{1}} \left[ \mathbb{E}_{\xi_i^{1} | \xi_i^{1}} \left\| \nabla F_i(\mathbf{w}^{1}) - \mathbf{h}_i^{2} \right\|^2 \right]\label{eq:lemma9-101}\\
&= \frac{1}{N} \sum_{i=1}^{N} \left[ p_i^t\, \mathbb{E}_{\xi_i^{1}} \left\| \nabla F_i(\mathbf{w}^{1}) - \mathbf{g}_i^{1} \right\|^2 + (1 - p_i^t) \left\| \nabla F_i(\mathbf{w}^{1}) - \mathbf{h}_i^{1} \right\|^2 \right]\label{eq:lemma9-102}\\
&= \frac{1}{N} \sum_{i=1}^{N} p_i^t\, \mathbb{E}_{\xi_i^{1}} \left\| \mathbf{g}_i^{1} - \bar{\mathbf{g}}_i^{1} \right\|^2 + \frac{1}{N} \sum_{i=1}^{N} p_i^t\, \mathbb{E}_{\xi_i^{1}} \left\| \nabla F_i(\mathbf{w}^{1}) - \bar{\mathbf{g}}_i^{1} \right\|^2 + \frac{1}{N} \sum_{i=1}^{N} (1 - p_i^t) \left\| \nabla F_i(\mathbf{w}^{1}) - \mathbf{h}_i^{1} \right\|^2\label{eq:lemma9-103}
\end{align}
where Eq. \eqref{eq:lemma9-101} uses the law of total expectation to separate expectations; Eq. \eqref{eq:lemma9-102} solves the inner expectation with respect to client participation; Eq. \eqref{eq:lemma9-103} adds and subtracts $\bar{g}_i^{1}$ to the first term, then leverages the local pseudo-gradients' unbiased property to separate the two components. 
\end{proof}

\textbf{Per Round Progress}

Define Lyapunov function, for any $\frac{\eta_s^2 L}{2} <\delta <\frac{\eta_s}{2}$ and $\alpha\geq0$: 
\begin{align}
\psi^{t+1}=F(w^{t+1})+(\delta-\frac{\eta_s^2 L}{2})\|\Delta^t\|^2+\alpha \underbrace{\frac{1}{N} \sum_{i=1}^N \|\nabla F_i(w^t)-h_i^{t+1}\|^2}_{H^{t+1}}
\end{align}
Considering expectation over the randomness at the $t$-th round and invoking the standard descent lemma for smooth objective: 
\begin{align}
&\mathbb{E}_{\mathcal{A}_t, \xi^{t} \mid \mathcal{H}^{t}} \left[ \psi^{t+1} \right] \\
&= \mathbb{E}_{\mathcal{A}_t, \xi^{t} \mid \mathcal{H}^{t}} \left[ F(w^{t+1}) + \left(\delta - \frac{\eta_s^2 L}{2}\right) \left\| \Delta^{t} \right\|^2 + \frac{\alpha}{N} \sum_{i=1}^{N} \left\| \nabla F_i(w^{t}) - h_i^{t+1} \right\|^2 \right]\\
&\leq F(w^{t}) - \frac{\eta_s}{2} \left\| \nabla F(w^{t}) \right\|^2 - \frac{\eta_s}{2} \, \mathbb{E}_{\xi^{t} \mid \mathcal{H}^{t}} \left\| \bar{g}^{t} \right\|^2 + \frac{\eta_s}{2} \, \mathbb{E}_{\xi^{t} \mid \mathcal{H}^{t}} \left\| \bar{g}^{t} - \nabla F(w^{t}) \right\|^2\\
&+ \frac{\eta_s^2 L}{2} \, \mathbb{E}_{\mathcal{A}_t, \xi^{t} \mid \mathcal{H}^{t}} \left\| \Delta^{t} \right\|^2 + \left(\delta - \frac{\eta_s^2 L}{2}\right) \mathbb{E}_{\mathcal{A}_t, \xi^{t} \mid \mathcal{H}^{t}} \left\| \Delta^{t} \right\|^2 + \frac{\alpha}{N} \sum_{i=1}^{N} \mathbb{E}_{\mathcal{A}_{t}, \xi_i^{t} \mid \mathcal{H}^{t}} \left\| \nabla F_i(w^{t}) - h_i^{t+1} \right\|^2\\
&\leq F(w^{t}) - \frac{\eta_s}{2} \left\| \nabla F(w^{t}) \right\|^2 - \frac{\eta_s}{2} \, \mathbb{E}_{\xi^{t} \mid \mathcal{H}^{t}} \left\| \bar{g}^{t} \right\|^2 + \frac{\eta_s}{2N} \sum_{i=1}^{N} \mathbbm{1}_{i\in\mathcal{A}_t} \, \mathbb{E}_{\xi_i^{t} \mid \mathcal{H}^{t}} \left\| \bar{g}_i^{t} - \nabla F_i(w^{t}) \right\|^2\\
&+ \delta \, \mathbb{E}_{\mathcal{A}_t, \xi^{t} \mid \mathcal{H}^{t}} \left\| \Delta^{t} - \bar{g}^{t} + \bar{g}^{t} \right\|^2 + \frac{\alpha}{N} \sum_{i=1}^{N} \mathbb{E}_{\mathcal{A}_{t}, \xi_i^{t} \mid \mathcal{H}^{t}} \left\| \nabla F_i(w^{t}) - h_i^{t+1} \right\|^2\\
&\leq F(w^{t}) - \frac{\eta_s}{2} \left\| \nabla F(w^{t}) \right\|^2 + \left(\delta - \frac{\eta_s}{2}\right) \mathbb{E}_{\xi^{t} \mid \mathcal{H}^{t}} \left\| \bar{g}^{t} \right\|^2 + \frac{\eta_s}{2N} \sum_{i=1}^{N}  \mathbb{E}_{\xi_i^{t} \mid \mathcal{H}^{t}} \left\| \bar{g}_i^{t} - \nabla F_i(w^{t}) \right\|^2\\
&+ \delta \, \mathbb{E}_{\mathcal{A}_t, \xi^{t} \mid \mathcal{H}^{t}} \left\| \Delta^{t} - \bar{g}^{t} \right\|^2 + \alpha \, \mathbb{E}_{\mathcal{A}_t, \xi^{t} \mid \mathcal{H}^{t}} \left[ \frac{1}{N} \sum_{i=1}^{N} \left\| \nabla F_i(w^{t}) - h_i^{t+1} \right\|^2 \right]\\
&\leq F(w^{t}) - \frac{\eta_s}{2} \left\| \nabla F(w^{t}) \right\|^2 + \frac{\eta_s}{2N} \sum_{i=1}^{N}  \mathbb{E}_{\xi_i^{t} \mid \mathcal{H}^{t}} \left\| \bar{g}_i^{t} - \nabla F_i(w^{t}) \right\|^2\\
&+ \delta \, \mathbb{E}_{\mathcal{A}_t, \xi^{t} \mid \mathcal{H}^{t}} \left\| \Delta^{t} - \bar{g}^{t} \right\|^2 + \alpha \, \mathbb{E}_{\mathcal{A}_t, \xi^{t} \mid \mathcal{H}^{t}} \left[ H^{t+1} \right]
\end{align}
We simplify some notations: 
\begin{align}
\mathbb{E}[\psi^{t+1}]
&\leq
F(w^t)-\frac{\eta_s}{2}\|\nabla F(w^t)\|^2
+\frac{\eta_s}{2N}\sum_{i=1}^N \mathbb{E}\|\overline{g}_i^t - \nabla F_i(w^t)\|^2 \label{eq:psi_ineq} \\
&+ \delta \underbrace{\mathbb{E}\|\Delta^t-\overline{g}^t\|^2}_{\text{variance of update}} +\alpha \mathbb{E}[H^{t+1}]
\end{align}
Next we apply Lemma \ref{lemma:memeory}, 
\begin{align}
\mathbb{E}[\psi^{t+1}]
&\leq
F(w^t)-\frac{\eta_s}{2}\|\nabla F(w^t)\|^2
+\frac{\eta_s}{2N}\sum_{i=1}^N \mathbb{E}\|\overline{g}_i^t - \nabla F_i(w^t)\|^2+ \delta\mathbb{E}\|\Delta^t-\overline{g}^t\|^2 \\
&+\alpha [p_{avg} \frac{\sigma^2}{K}+\frac{1}{N} \sum_{i=1}^N p_i^t\mathbb{E}\|\overline{g}_i^t-\nabla F_i(w^t)\|^2\\
&+\eta_s^2L^2(1+\frac{1}{C})(1-p_{avg})\|\Delta^{t-1}\|^2+(1+C)(1-p_{min})H^t]\\
&=F(w^t)\\
&+\alpha \eta_s^2L^2(1+\frac{1}{C})(1-p_{avg})\|\Delta^{t-1}\|^2+\alpha (1+C)(1-p_{min})H^t
\\
&-\frac{\eta_s}{2}\|\nabla F(w^t)\|^2+\frac{\eta_s}{2N}\sum_{i=1}^N \mathbb{E}\|\overline{g}_i^t - \nabla F_i(w^t)\|^2+ \delta\mathbb{E}\|\Delta^t-\overline{g}^t\|^2 
\\
&+\alpha p_{avg} \frac{\sigma^2}{K}+\frac{\alpha}{N} \sum_{i=1}^N p_i^t\mathbb{E}\|\overline{g}_i^t-\nabla F_i(w^t)\|^2 \label{eq:psi}
\end{align}
For bounding within the Lyapunov recursive framework, the conditions for this recursion step are: 
\begin{align}
\alpha \eta_s^2L^2(1+\frac{1}{C})(1-p_{avg})&\leq \delta -\frac{\eta_s^2 L}{2} \label{eq:alpha2}\\
\alpha (1+C)(1-p_{min}) &\leq \alpha  \label{eq:alpha1}
\end{align}
From Eq. \eqref{eq:alpha1}, we know
\begin{align}
C\leq \frac{p_{min}}{1-p_{min}}
\end{align}
We select $C=\frac{p_{min}}{2(1-p_{min})}$. From Eq. \eqref{eq:alpha2}, we can have many selections of $\alpha, \delta$. For the convenience of writing, we select $\delta=\eta_s^2L$. Then, $\alpha$ needs: 
\begin{align}
\alpha \leq \frac{1}{2L(1+\frac{1}{C})(1-p_{avg})}=\frac{p_{min}}{2L(2-p_{min})(1-p_{avg})}
\end{align}
We select $\alpha = \frac{p_{min}}{4L(2-p_{min})}$. Using the values of $\delta$ and $\alpha$ we select, Eq. \eqref{eq:psi} can be written as: 
\begin{align}
\mathbb{E}[\psi^{t+1}]&\leq \psi^t
-\frac{\eta_s}{2}\|\nabla F(w^t)\|^2\\
&+\frac{\eta_s}{2N}\sum_{i=1}^N \mathbb{E}\|\overline{g}_i^t - \nabla F_i(w^t)\|^2+ \delta\mathbb{E}\|\Delta^t-\overline{g}^t\|^2 \\
&+\alpha p_{avg} \frac{\sigma^2}{K}+\frac{\alpha}{N} \sum_{i=1}^N p_i^t\mathbb{E}\|\overline{g}_i^t-\nabla F_i(w^t)\|^2\\
&=\psi^t
-\frac{\eta_s}{2}\|\nabla F(w^t)\|^2\\
&+\frac{\eta_s}{2N}\sum_{i=1}^N \mathbb{E}\|\overline{g}_i^t - \nabla F_i(w^t)\|^2+ \eta_s^2L \mathbb{E}\|\Delta^t-\overline{g}^t\|^2 \\
&+\frac{p_{min} p_{avg} \sigma^2}{4L(2-p_{min})K} \\
&+\frac{p_{min}}{4L(2-p_{min})N} \sum_{i=1}^N p_i^t\mathbb{E}\|\overline{g}_i^t-\nabla F_i(w^t)\|^2
\end{align}
Use Lemma \ref{lemma:5} to bound $\mathbb{E}\|\overline{g}_i^t-\nabla F_i(w^t)\|^2$, we can further have: 
\begin{align}
\mathbb{E}[\psi^{t+1}]
&\leq \psi^t
-\frac{\eta_s}{2}\|\nabla F(w^t)\|^2\\
&+\frac{\eta_s}{2N}\sum_{i=1}^N \mathbb{E}\|\overline{g}_i^t - \nabla F_i(w^t)\|^2+ \eta_s^2L \mathbb{E}\|\Delta^t-\overline{g}^t\|^2 \\
&+\frac{p_{min} p_{avg} \sigma^2}{4L(2-p_{min})K} \\
&+\frac{p_{min}}{4L(2-p_{min})N} \sum_{i=1}^N p_i^t\mathbb{E}\|\overline{g}_i^t-\nabla F_i(w^t)\|^2\\
&\leq 
\psi^t+ \eta_s^2L \mathbb{E}\|\Delta^t-\overline{g}^t\|^2
-\frac{\eta_s}{2}\|\nabla F(w^t)\|^2
\\
&+[\frac{\eta_s}{2}+\frac{p_{min} p_{avg}}{4L(2-p_{min})}] 8\eta_c^2L^2K(K-1) \|\nabla F(w^t)\|^2\\
&+\frac{p_{min} p_{avg} \sigma^2}{4L(2-p_{min})K} \\
&+[\frac{\eta_s}{2}+\frac{p_{min} p_{avg}}{4L(2-p_{min})}] 2\eta_c^2L^2K(K-1) \frac{\sigma^2}{K}\\
&+[\frac{\eta_s}{2}+\frac{p_{min} p_{avg}}{4L(2-p_{min})}] 8\eta_c^2L^2K(K-1) \sigma_g^2 
\end{align}
Here we want the coefficient for the gradient squared norm not exceed $\frac{-\eta_s}{4}$. To achieve this, we bound learning rates as: 
\begin{align}
8\eta_c^2L^2K(K-1)&\leq \frac{1}{4}\\
\frac{p_{min} p_{avg}}{4L(2-p_{min})} &\leq \frac{\eta_s}{2}
\end{align}
Requirements for learning rates can be easily derived here. Then, we have
\begin{align}
\mathbb{E}[\psi^{t+1}]
&\leq 
\psi^t+ \eta_s^2L \mathbb{E}\|\Delta^t-\overline{g}^t\|^2
-\frac{\eta_s}{4}\|\nabla F(w^t)\|^2 \label{eq:lemma10} 
\\
&+\frac{p_{min} p_{avg} \sigma^2}{4L(2-p_{min})K} \\
&+[\frac{\eta_s}{2}+\frac{p_{min} p_{avg}}{4L(2-p_{min})}] 2\eta_c^2L^2K(K-1) \frac{\sigma^2}{K}\\
&+[\frac{\eta_s}{2}+\frac{p_{min} p_{avg}}{4L(2-p_{min})}] 8\eta_c^2L^2K(K-1) \sigma_g^2 \label{eq:psi}
\end{align}
Next, we provide a bound for the initial progress following similar steps before. 

The initial progress can be bounded following similar steps above. Simlarly as Eq. \eqref{eq:psi_ineq}, 
\begin{align}
\mathbb{E}[\psi^2]&\leq F(w^1)-\frac{\eta_s}{2} \|\nabla F(w^1)\|^2+\frac{\eta_s}{2N} \sum_{i=1}^N \mathbb{E}\|\overline{g}_i^t - \nabla F_i(w^1) \|^2\\
&+\delta \mathbb{E}\|\Delta^1-\overline{g}^1\|^2+\alpha \mathbb{E}[H^2]
\end{align}
where $H^2=\frac{1}{N}\sum_{i=1}^N \|\nabla F_i(w^1)-h_i^2\|^2$. 
Similarly, using Lemma \ref{lemma:5}, we can bound this as: 
\begin{align}
\mathbb{E}[\psi^2]&\leq F(w^1)-\frac{\eta_s}{4} \|\nabla F(w^1)\|^2
+\eta_s^2 L \mathbb{E}\|\Delta^1-\overline{g}^1\|^2\\
&+\frac{p_{min} p_{avg} \sigma^2}{4L(2-p_{min})L}\\
&+[\frac{\eta_s}{2}+\frac{p_{min} p_{avg}}{4L(2-p_{min})}] 2\eta_c^2 L^2 K(K-1) \frac{\sigma^2}{K}\\
&+[\frac{\eta_s}{2}+\frac{p_{min} p_{avg}}{4L(2-p_{min})}] 8\eta_c^2 L^2 K(K-1) \sigma_g^2\\
&+\frac{p_{min}(1-p_{min})}{4L(2-p_{min})} \frac{1}{N} \sum_{i=1}^N \|\nabla F_i(w^1)-h_i^1 \|^2
\end{align}
Finally, we can get the convergence conclusion. 
From Eq. \eqref{eq:lemma10}, unfolding he recursion for $t=2,3,\cdots, T$, it yields: 
\begin{align}
\mathbb{E}[\psi^{t+1}]
&\leq
\psi^2 + \eta_s^2 L \sum_{t=2}^T \mathbb{E}\|\Delta^t-\overline{g}^t\|^2
-\frac{\eta_s}{4}\sum_{t=2}^T \mathbb{E}\|\nabla F(w^t)\|^2\\
&+\sum_{t=2}^T \frac{p_{min} p_{avg} \sigma^2}{4L(2-p_{min})K} \\
&+\sum_{t=2}^T[\frac{\eta_s}{2}+\frac{p_{min} p_{avg}}{4L(2-p_{min})}] 2\eta_c^2L^2K(K-1) \frac{\sigma^2}{K}\\
&+\sum_{t=2}^T[\frac{\eta_s}{2}+\frac{p_{min} p_{avg}}{4L(2-p_{min})}] 8\eta_c^2L^2K(K-1) \sigma_g^2 
\end{align}
$\psi^2$ can be bounded as Eq. \eqref{eq:psi}, therefore, we can get
\begin{align}
\mathbb{E}[\psi^{t+1}]
&\leq
F(w^1) + \eta_s^2 L \sum_{t=1}^T \mathbb{E}\|\Delta^t-\overline{g}^t\|^2
-\frac{\eta_s}{4}\sum_{t=1}^T \mathbb{E}\|\nabla F(w^t)\|^2\\
&+\sum_{t=1}^T \frac{p_{min} p_{avg} \sigma^2}{4L(2-p_{min})K} \\
&+\sum_{t=1}^T[\frac{\eta_s}{2}+\frac{p_{min} p_{avg}}{4L(2-p_{min})}] 2\eta_c^2L^2K(K-1) \frac{\sigma^2}{K}\\
&+\sum_{t=1}^T[\frac{\eta_s}{2}+\frac{p_{min} p_{avg}}{4L(2-p_{min})}] 8\eta_c^2L^2K(K-1) \sigma_g^2 \\
&+\frac{p_{min}(1-p_{min})}{4L(2-p_{min})} \frac{1}{N} \sum_{i=1}^N \|\nabla F_i(w^1)-h_i^1 \|^2 \label{eq:sumT}
\end{align}
Notice that $\min_{t\in[1,T]} \mathbb{E}\|\nabla F(w^t)\|^2 \leq \frac{\sum_{t=1}^T\|\nabla F(w^t)\|^2 }{T}$
Dividing both sides of Eq. \eqref{eq:sumT} by T, after rearranging the terms, we can get
\begin{align}
\min_{t\in[1,T]} \mathbb{E}\|\nabla F(w^t)\|^2
&\leq
\frac{4(F(w^1)-\mathbb{E}[\psi^T])}{\eta_s T} + 
4 \eta_s L \frac{\sum_{t=1}^T \mathbb{E}\|\Delta^t-\overline{g}^t\|^2}{T}
\\
&+ \frac{1}{T}\frac{p_{min} p_{avg} \sigma^2}{\eta_s L(2-p_{min})K} \\
&+[2+\frac{p_{min} p_{avg}}{\eta_s L(2-p_{min})}] 2\eta_c^2L^2K(K-1) \frac{\sigma^2}{K}\\
&+[2+\frac{p_{min} p_{avg}}{\eta_s L(2-p_{min})}] 8\eta_c^2L^2K(K-1) \sigma_g^2 \\
&+\frac{1}{T} \frac{p_{min}(1-p_{min})}{\eta_s L(2-p_{min})} \frac{1}{N} \sum_{i=1}^N \|\nabla F_i(w^1)-h_i^1 \|^2
\end{align}
Notice $\psi^T \geq F(w^T) \geq F(w^*)$. Therefore, we can have
\begin{align}
\min_{t\in[1,T]} \mathbb{E}\|\nabla F(w^t)\|^2
&\leq
\underbrace{\frac{4(F(w^1)-F(w^*))}{\eta_s T}}_{\text{iterate initialization error}} + 
\underbrace{4 \eta_s L \frac{\sum_{t=1}^T \mathbb{E}\|\Delta^t-\overline{g}^t\|^2}{T}}_{\text{partial update variance error}}
\\
&+ \underbrace{\frac{1}{T}\frac{p_{min} p_{avg} \sigma^2}{\eta_s L(2-p_{min})K}}_{\text{stochastic gradient error (will disappear)}} \\
&+\underbrace{[2+\frac{p_{min} p_{avg}}{\eta_s L(2-p_{min})}] 2\eta_c^2L^2K(K-1) \frac{\sigma^2}{K}}_{\text{stochastic gradient error}}\\
&+\underbrace{[2+\frac{p_{min} p_{avg}}{\eta_s L(2-p_{min})}] 8\eta_c^2L^2K(K-1) \sigma_g^2}_{\text{error from data heterogeneity (non-iid level)}} \\
&+\underbrace{\frac{1}{T} \frac{p_{min}(1-p_{min})}{\eta_s L(2-p_{min})} \frac{1}{N} \sum_{i=1}^N \|\nabla F_i(w^1)-h_i^1 \|^2}_{\text{memory initialization error}}
\end{align} 

% For the partial participation variance error, easy to see that we have a lower bound compared to FedVARP. 

\section{Additional Experimental Details}\label{app:exp}
\subsection{Detailed Hyperparameters and Model Structures}\label{app:setup}
For reproducibility, we provide a detailed table (Table \ref{tab:hyperparams}) specifying the exact model architectures and hyperparameters that we used in the experiments. 
\begin{table}[h]
\centering
\caption{Hyperparameter settings across datasets.}
\label{tab:hyperparams}
\resizebox{0.99\linewidth}{!}{
\begin{tabular}{lccc}
\toprule
\textbf{Hyperparameters} & \textbf{EMNIST} & \textbf{Fashion-MNIST} & \textbf{CIFAR-10} \\
\midrule
Total rounds        & 150 & 150 & 150 \\
Local batch size    & 64  & 64  & 64  \\
Local epochs        & 5   & 5   & 5   \\
Local learning rate & 0.5 & 0.01 & 0.3 \\
Global learning rate& 0.5 & 0.5 & 0.5 \\
Model architecture  &
CNN (3 Conv: 16, 24, 32 filters, $5\times5$; 3 FC: 256, 84 units) &
CNN (2 Conv: 64/128 filters, $3\times3$; 2 FC: 512 units) &
PreAct ResNet-18 \cite{he2016resnet} \\
\bottomrule
\end{tabular}}
\end{table}

\subsection{Detailed Comparison of Communication and Compute Overhead}
\label{app:communication}
FedSteer is designed specifically for resource-constrained edge devices. We provide a detailed comparison below:  

\textbf{Communication Cost (Identical to FedVARP/FedAvg):}
FedSteer does not incur any additional communication costs compared to standard FL. Active clients send new updates; while inactive clients send nothing (they do not perform local training when they are inactive). No additional bits or control variates are transmitted during training. The dynamic subspace optimization happens on the server only.  

\textbf{Client Compute (Zero Additional Cost Compared to FedAvg):} FedSteer’s projection logic (solving $s_i$) happens entirely on the server. Clients only perform standard local training when it participates. 

\textbf{Server Compute (Negligible):}
The server solves a Ridge Regression problem for active clients (Eq. (5) in the paper) to derive the coordinates $\mathbf{s}_i$. The cost of directly computing the closed-form solution (Eq. (6)) is $O(k^2d+k^3)$, where $d$ is the model dimension and $k=|\mathcal{X}|$ is the core set size. Compared to the cost of standard weight aggregation ($O(N d)$), directly computing the closed-form may yield comparable cost (for example, N=100, k=10), normally the server is computationally powerful, so they can handle it. In the case that the server is also computation limited, it can alternatively solve it via gradient descent to avoid matrix operations entirely, as the problem is strictly convex. The core set selection (Algorithm 2) is restricted to a short "warm-up" phase, then the core set is fixed -- this overhead drops to zero. 

\textbf{Server Memory (Superior Efficiency):}
FedSteer is significantly more memory-efficient than FedVARP/FedStale/MIFA.
FedVARP/FedStale/MIFA must store full stale gradients for all $N$ clients ($O(N \cdot d)$).
FedSteer stores projection coordinates ($\mathbf{s}_i \in \mathbb{R}^{k}$) for all clients, and full stale gradients only for the small core set ($k \ll N$). Therefore the cost is $O((N \cdot k + k \cdot d)$. For large models ($d$) and many clients ($N$), this is a massive reduction. In experiments, we set $k=10, N=100$, resulting in reduced memory usage by around 10 times.

\subsection{Scalability via Subsampling}\label{app:subsampling}

\begin{table}[h]
\centering
\caption{Performance comparison under subsampled core-set selection on Fashion-MNIST.}
\label{tab:subsampling}
\begin{tabular}{lcc}
\toprule
\textbf{Methods} & \textbf{$\gamma=0.9$} & \textbf{$\gamma=0.7$} \\
\midrule
FedAvg & $0.509 \pm 0.010$ & $0.609 \pm 0.010$ \\
FedSteer ($N_{\text{sub}}=20$) & $0.537 \pm 0.013$ & $0.644 \pm 0.011$ \\
FedSteer ($N_{\text{sub}}=100$) & $0.554 \pm 0.008$ & $0.657 \pm 0.011$ \\
\bottomrule
\end{tabular}
\end{table}

To ensure scalability for very large $N$, we introduce a random subsampling strategy during core-set selection. 
Instead of searching over the entire client population, we randomly sample a subset $N_{\text{sub}} \ll N$ (e.g., 20 candidates out of 100) and perform the greedy swap search within this subset. 
This reduces the selection complexity to 
\(
O\!\left(T_{0} I_{\max} N_{\text{sub}} k \right).
\)
Since $N_{\text{sub}}$ is small and fixed (e.g., $N_{\text{sub}}=20$), the computational cost becomes independent of the total population size $N$, enabling scalability to large-scale FL systems.
We further evaluate this subsampling strategy. As shown in Table~\ref{tab:subsampling}, subsampling incurs only a minor accuracy drop (approximately 2.5\% on average across two settings) while still outperforming all baselines.

\subsection{Additional Experimental Results}\label{app:plots}

Figures \ref{fig:fashion} and \ref{fig:cifar} empirically demonstrate the convergence behavior of FedSteer against multiple baseline algorithms over 150 global training iterations under extreme heterogeneity ($\gamma=0.9$). Across both the Fashion-MNIST and CIFAR-10 tasks, FedSteer achieves stable convergence and significantly outperforms all evaluated baseline methods, obtaining a final accuracy superseded only by the theoretical ``full participation'' upper bound.

\begin{figure}[h]
    \centering
    \includegraphics[width=0.6\linewidth]{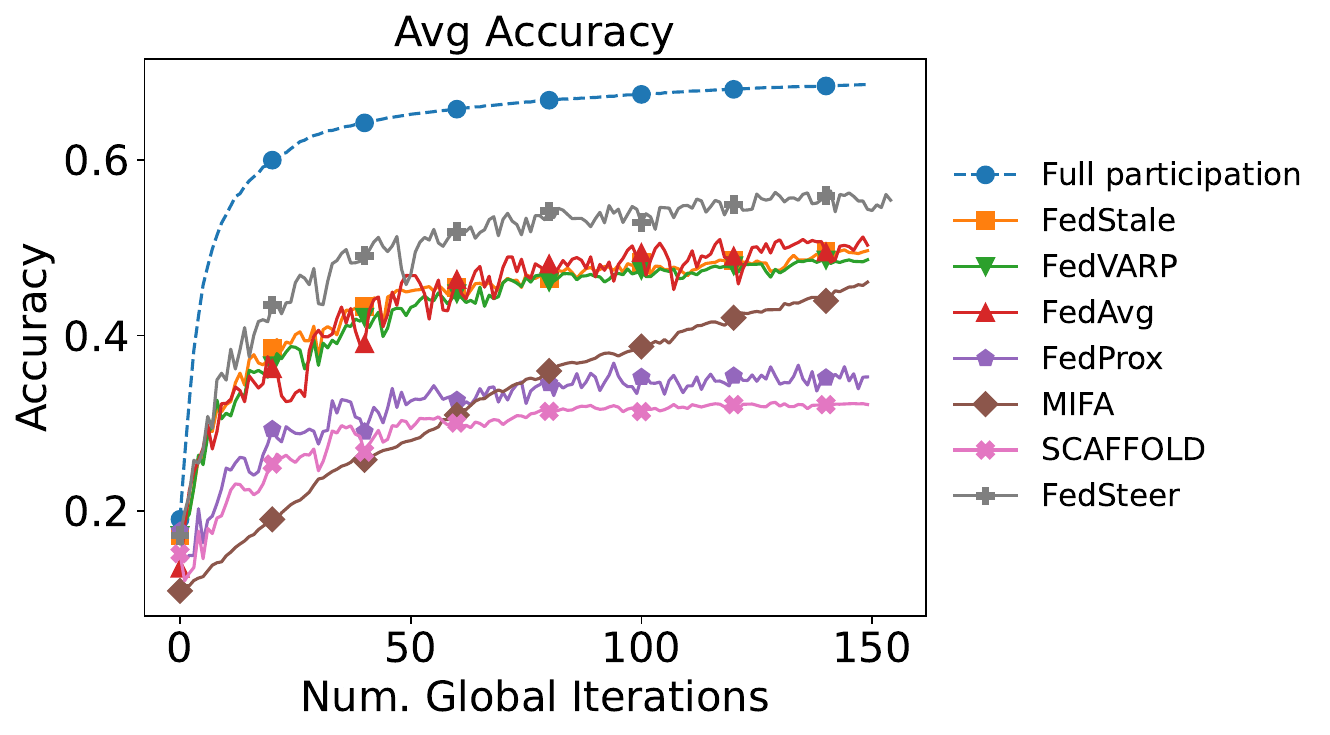}
    \caption{Comparison of average test accuracy against the number of global iterations on the Fashion-MNIST dataset under extreme data and system heterogeneity ($\gamma=0.9$).}
    \label{fig:fashion}
\end{figure}

\begin{figure}[h]
    \centering
    \includegraphics[width=0.6\linewidth]{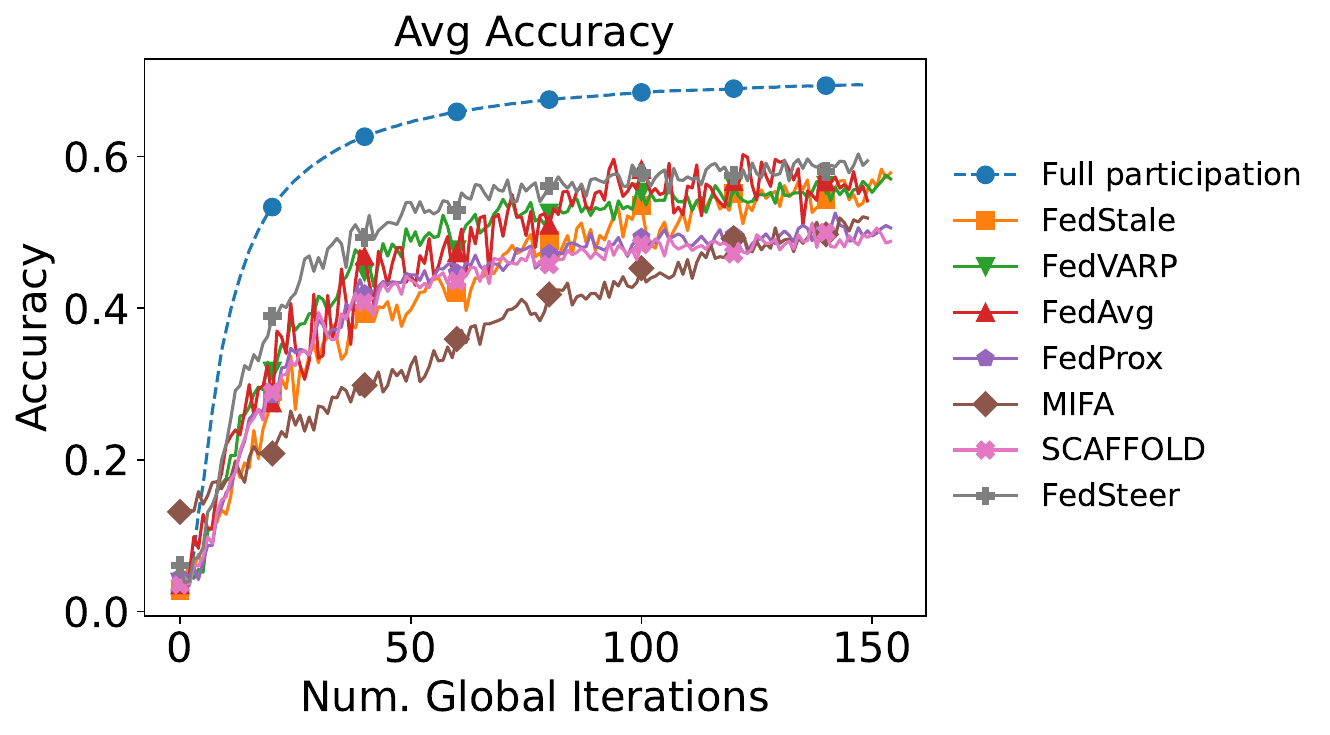}
    \caption{Comparison of average test accuracy against the number of global iterations on the CIFAR-10 dataset under extreme data and system heterogeneity ($\gamma=0.9$).}
    \label{fig:cifar}
\end{figure}

\subsection{Effectiveness of reusing $\mathbf{s}_i$ across rounds}\label{app:reusingsi}

Figure \ref{fig:main-comparison} illustrates the evolution of the projection coefficients, $s_i$, for three representative clients (5, 55, and 95) under different experimental conditions. We present a 2x2 comparison varying the regularization strength ($\lambda$) and the core set size ($|X|$). Each heatmap plots the coefficient vector $s_i$ (y-axis) against the communication round (x-axis). The most prominent observation across all settings is the \textbf{stability of these coefficients over time}. The vertical patterns, which represent the distribution of $s_i$ at a given round, exhibit slight changes as training progresses. This indicates that the optimal projection of a client's data onto the shared base vectors remains largely consistent. 
The result is conducted with \texttt{FedSteer} (\texttt{enforce=5}) under EMNIST. The observed stability supports FedSteer to reuse projection coefficients instead of the stale gradients, as the coefficients do not change significantly between rounds. Furthermore, the distinct patterns among clients (e.g., the contrast between client 55 and the others) highlight the method's ability to effectively capture stable, yet heterogeneous, client-specific representations.

\begin{figure}[h]
    \centering

    %----------- TOP ROW: Core Set Size 20 -----------
    \begin{subfigure}[b]{1.0\textwidth}
        \centering
        \includegraphics[width=0.32\linewidth]{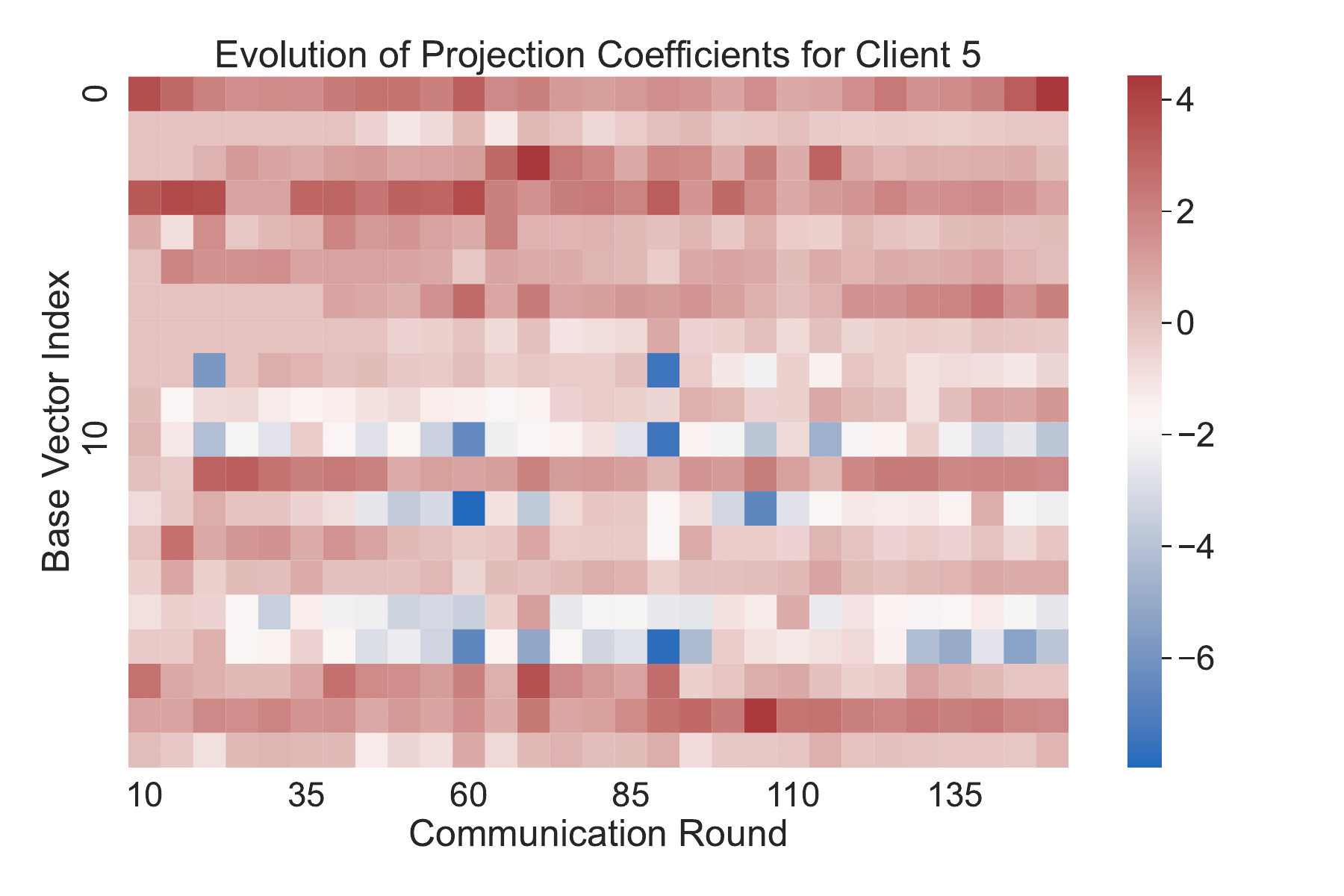}
        \includegraphics[width=0.32\linewidth]{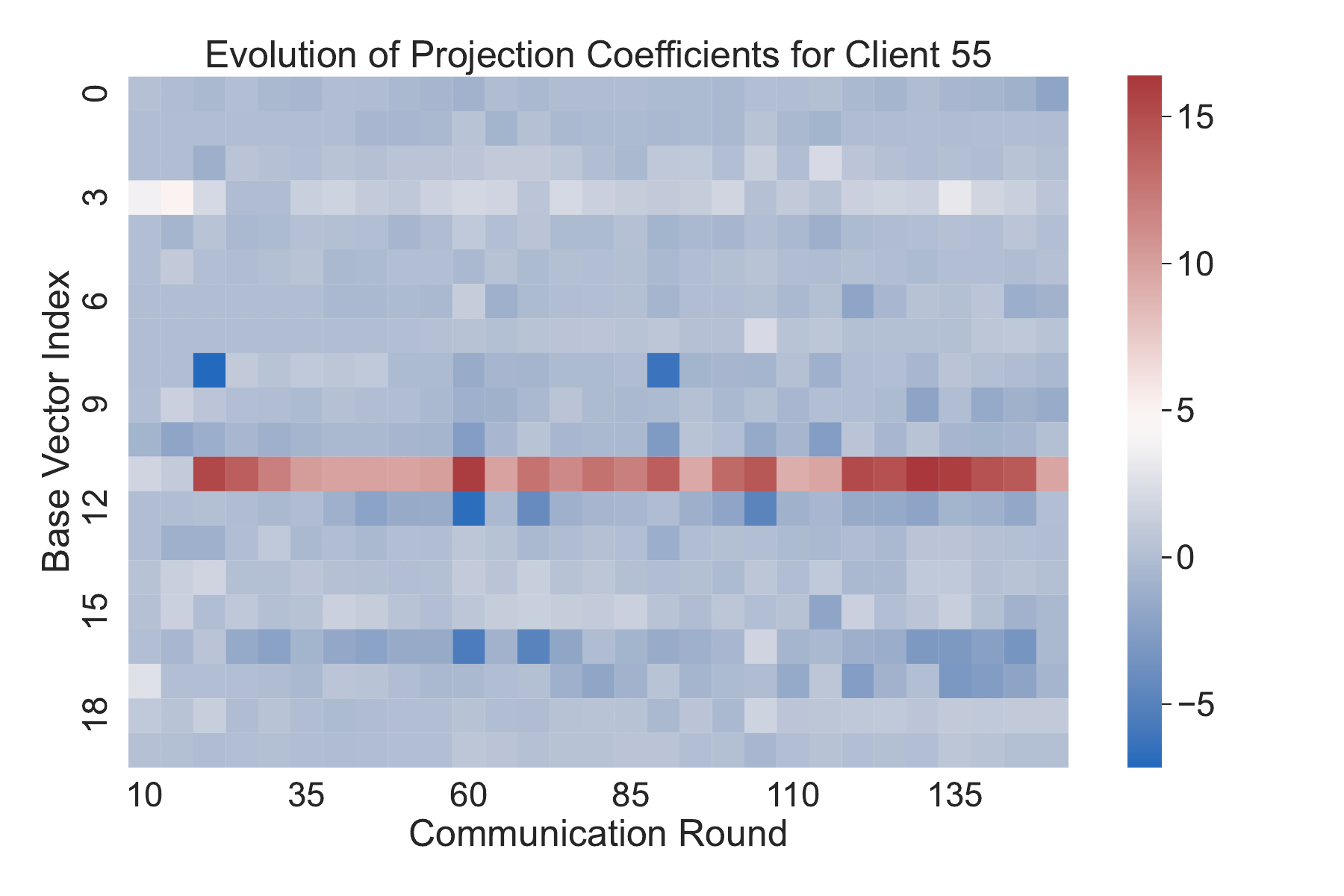}
        \includegraphics[width=0.32\linewidth]{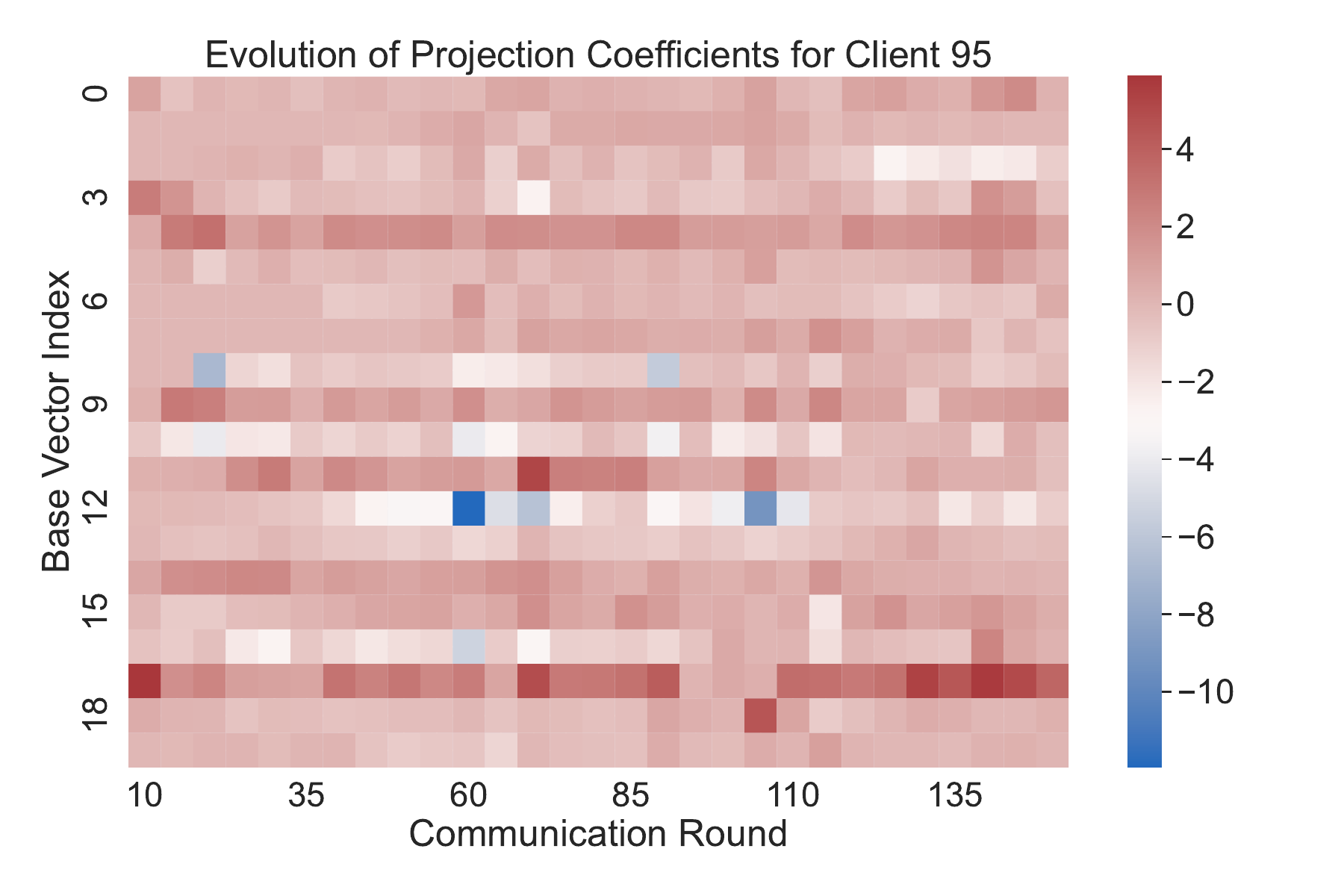}
        \caption{Without regularization ($\lambda=1\mathrm{e}{-6}$), $|\mathcal{X}|=20$.}
        \label{subfig:no-reg-20}
    \end{subfigure}
    \hfill % Adds horizontal space between the two top subfigures
    \begin{subfigure}[b]{1.0\textwidth}
        \centering
        \includegraphics[width=0.32\linewidth]{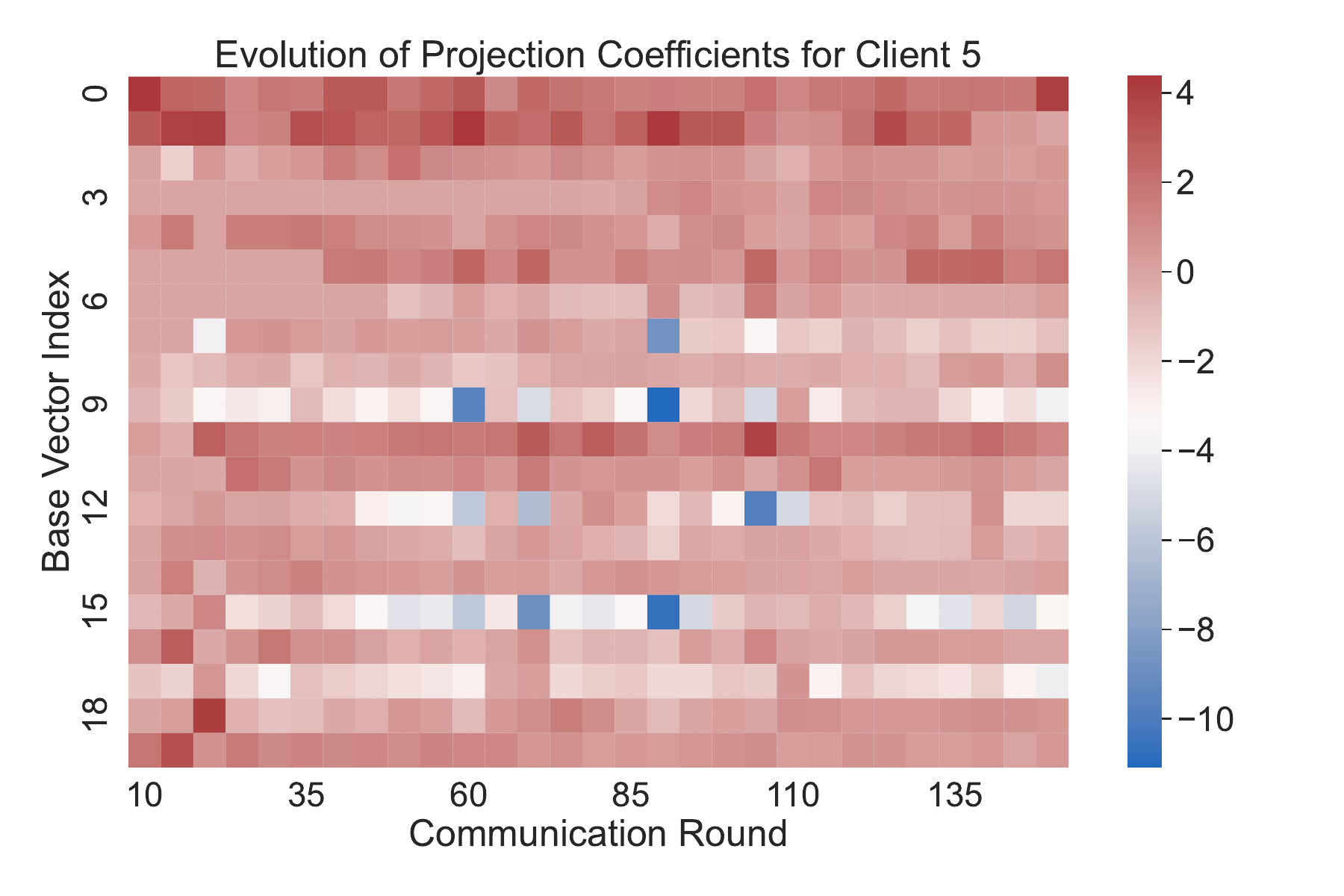}
        \includegraphics[width=0.32\linewidth]{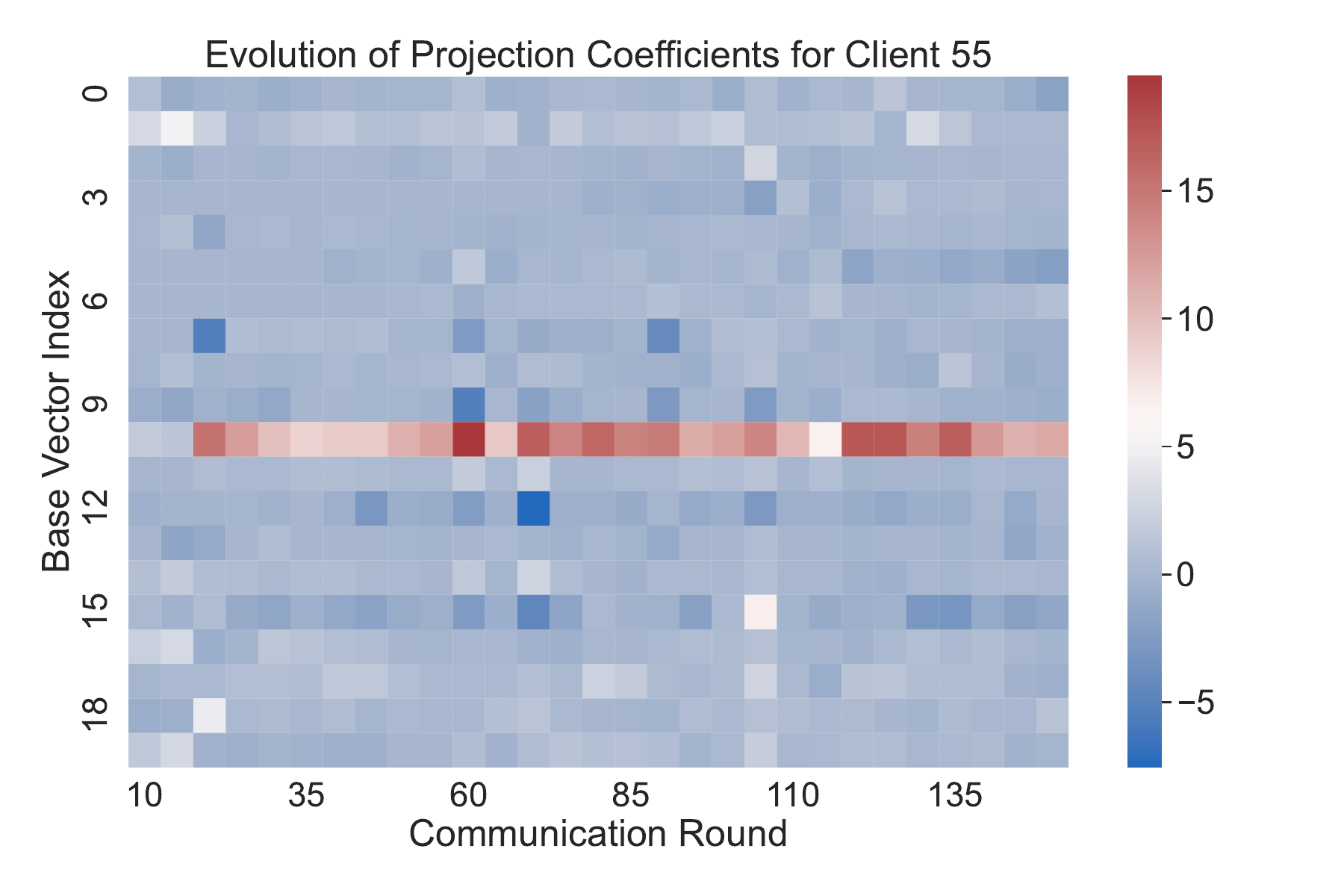}
        \includegraphics[width=0.32\linewidth]{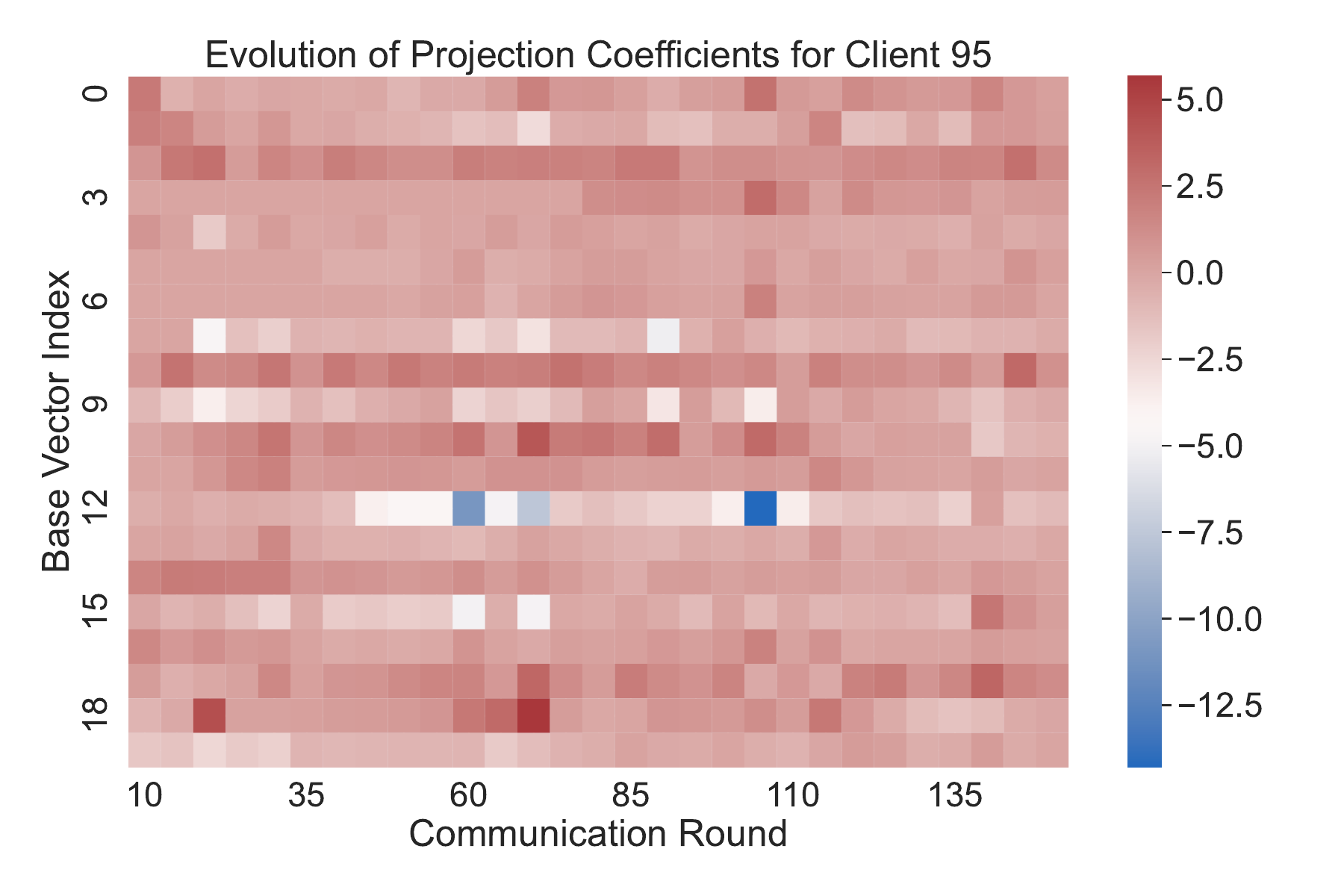}
        \caption{With regularization ($\lambda=0.5$), $|\mathcal{X}|=20$.}
        \label{subfig:reg-20}
    \end{subfigure}

    %----------- BOTTOM ROW: Core Set Size 80 -----------
    \begin{subfigure}[b]{1.0\textwidth}
        \centering
        \includegraphics[width=0.32\linewidth]{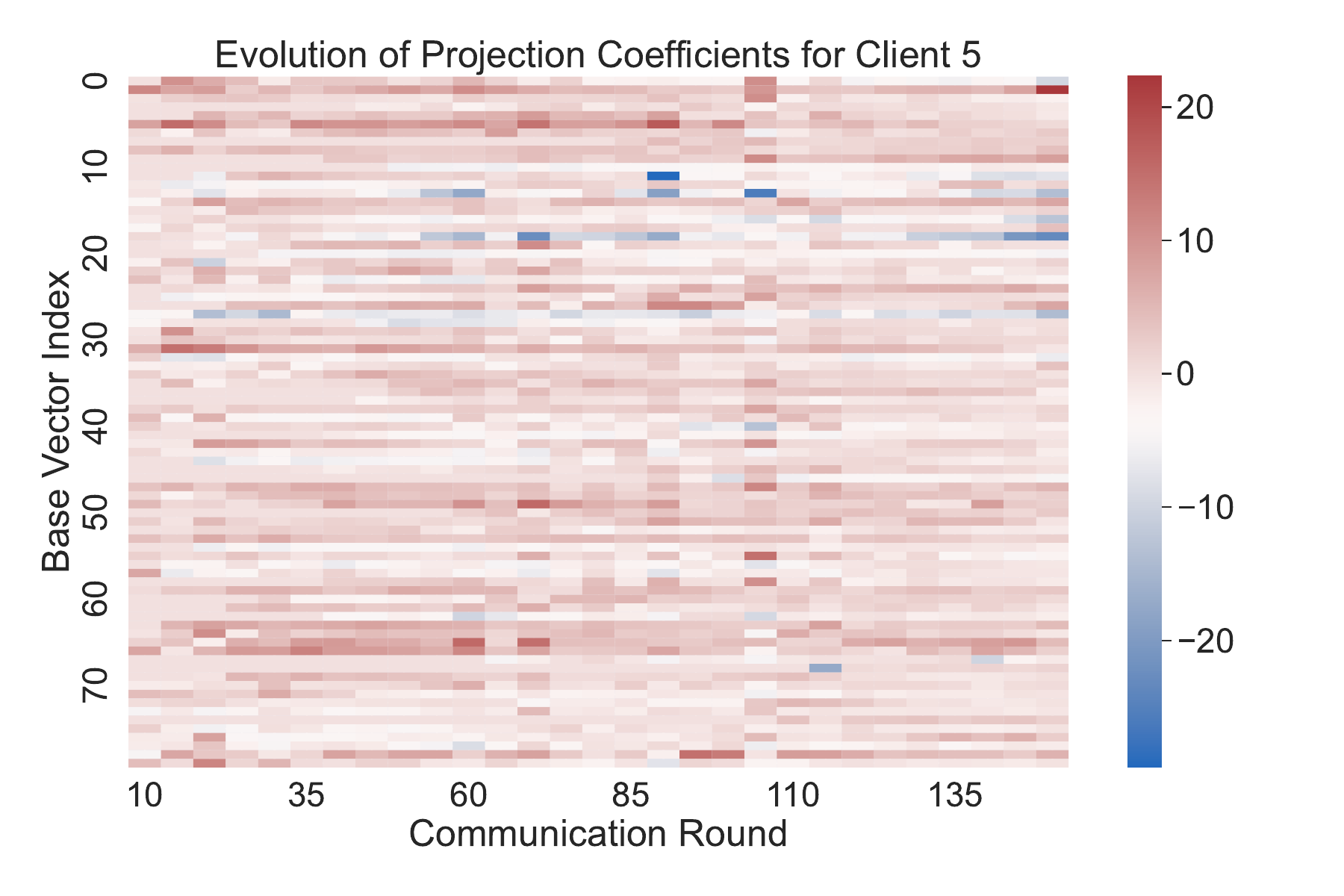}
        \includegraphics[width=0.32\linewidth]{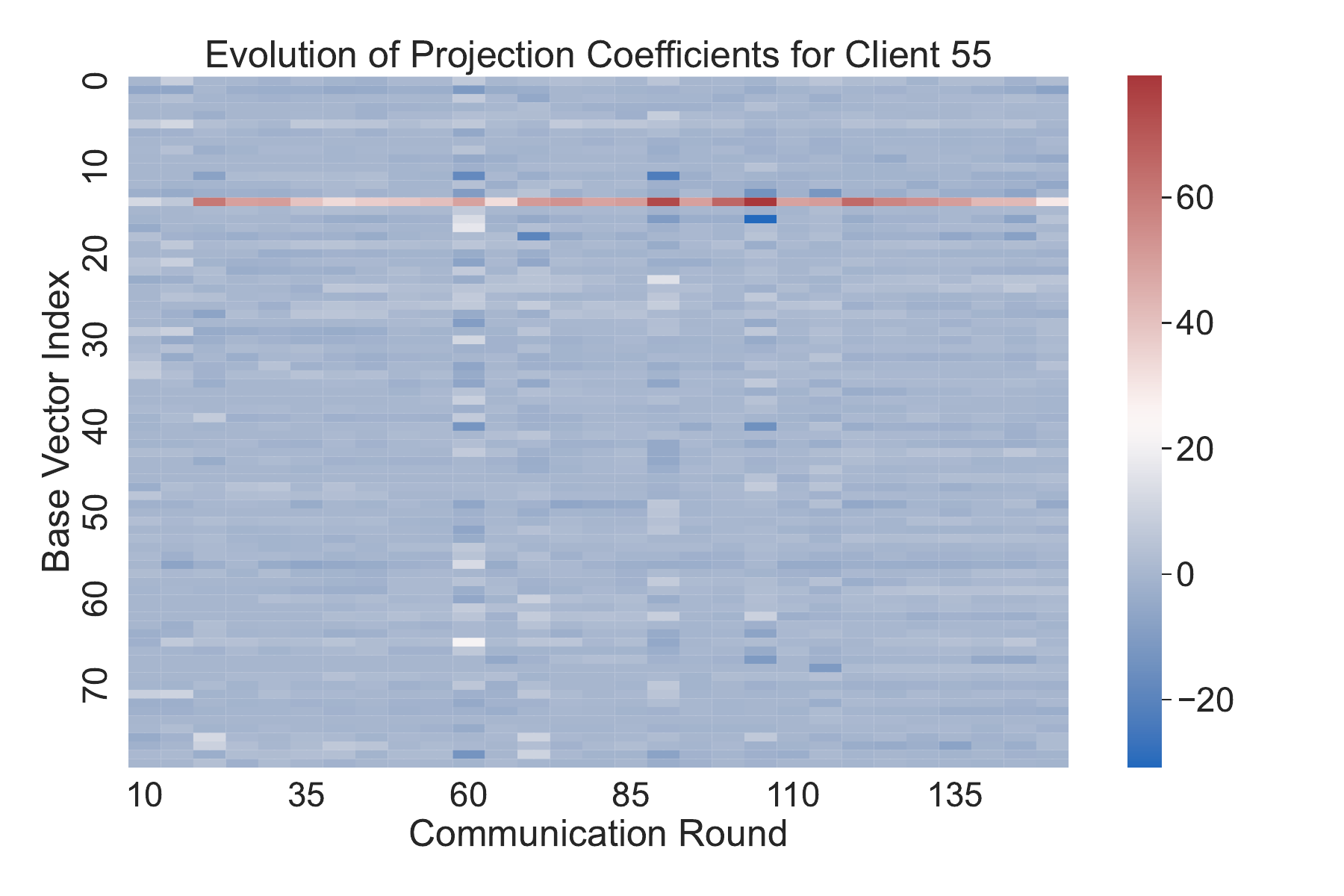}
        \includegraphics[width=0.32\linewidth]{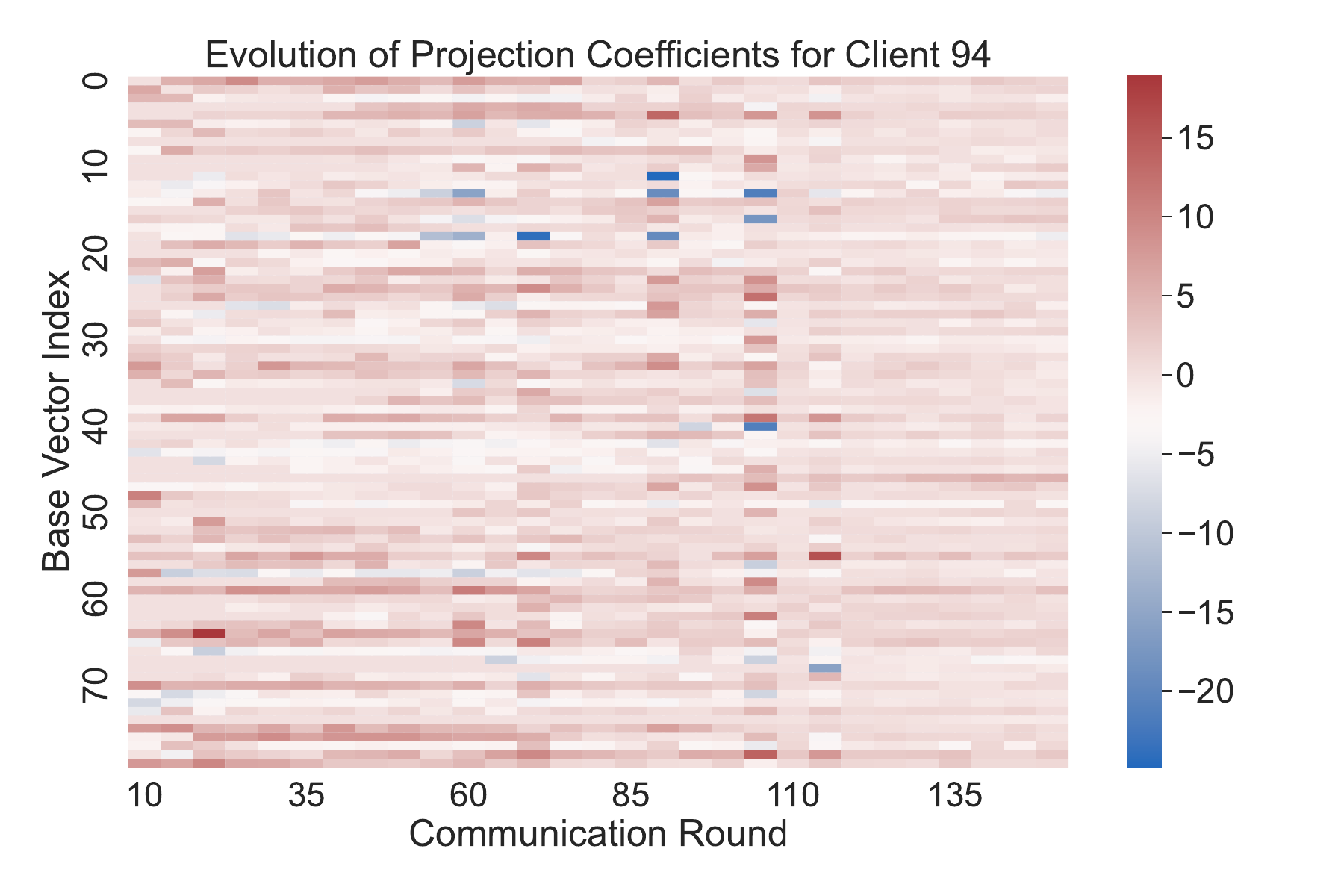}
        \caption{Without regularization ($\lambda=1\mathrm{e}{-6}$), $|\mathcal{X}|=80$.}
        \label{subfig:no-reg-80}
    \end{subfigure}
    \hfill % Adds horizontal space between the two bottom subfigures
    \begin{subfigure}[b]{1.0\textwidth}
        \centering
        \includegraphics[width=0.32\linewidth]{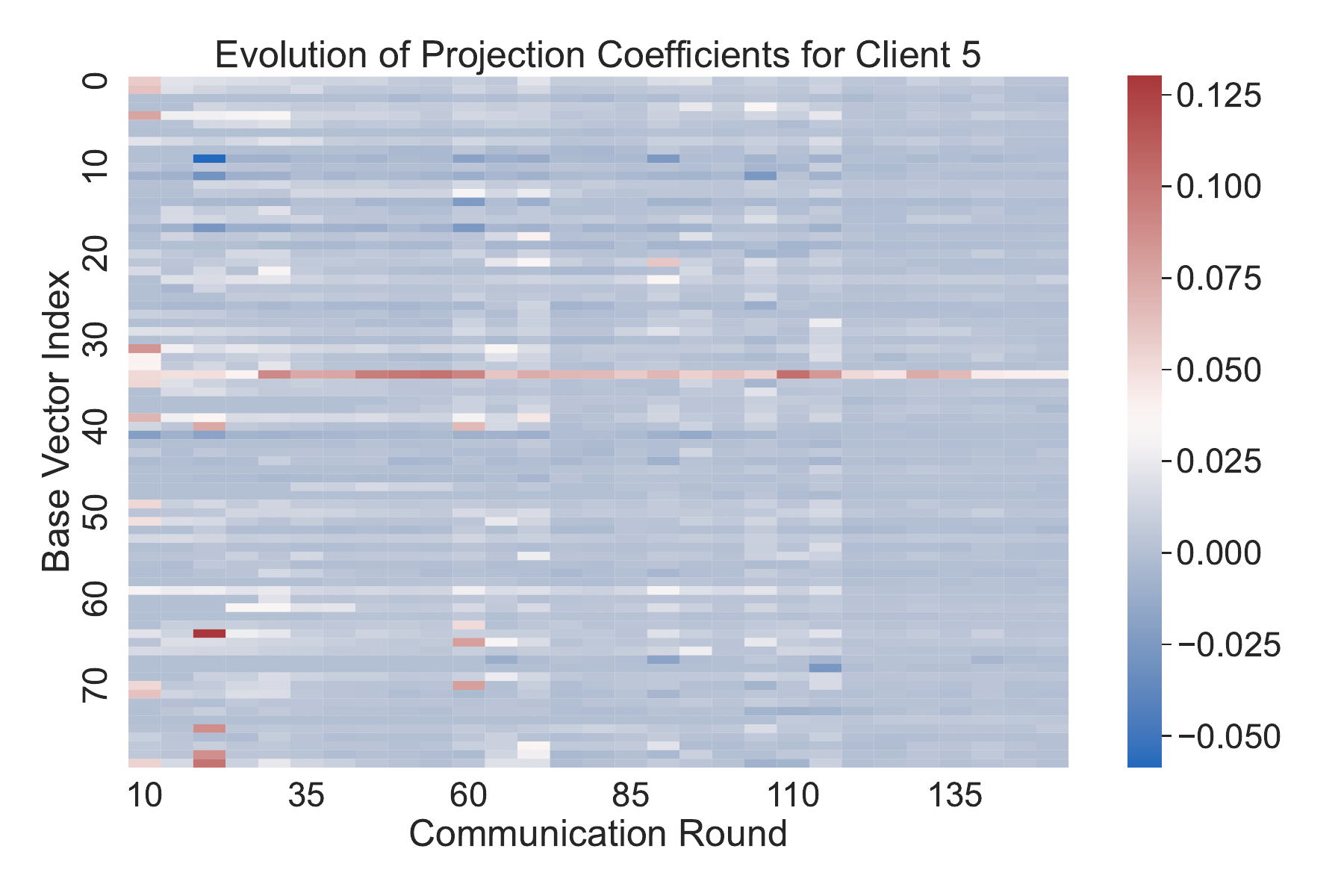}
        \includegraphics[width=0.32\linewidth]{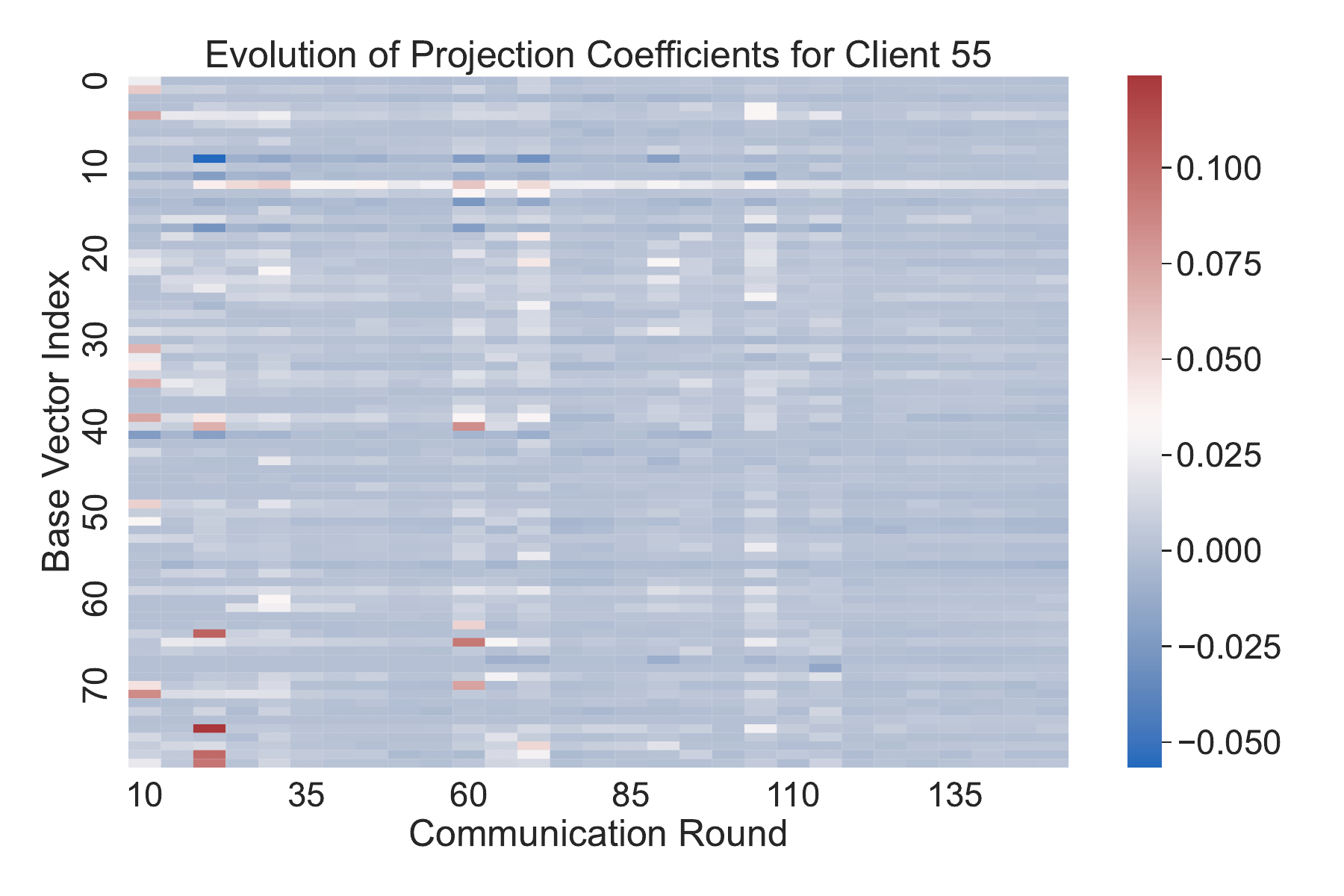}
        \includegraphics[width=0.32\linewidth]{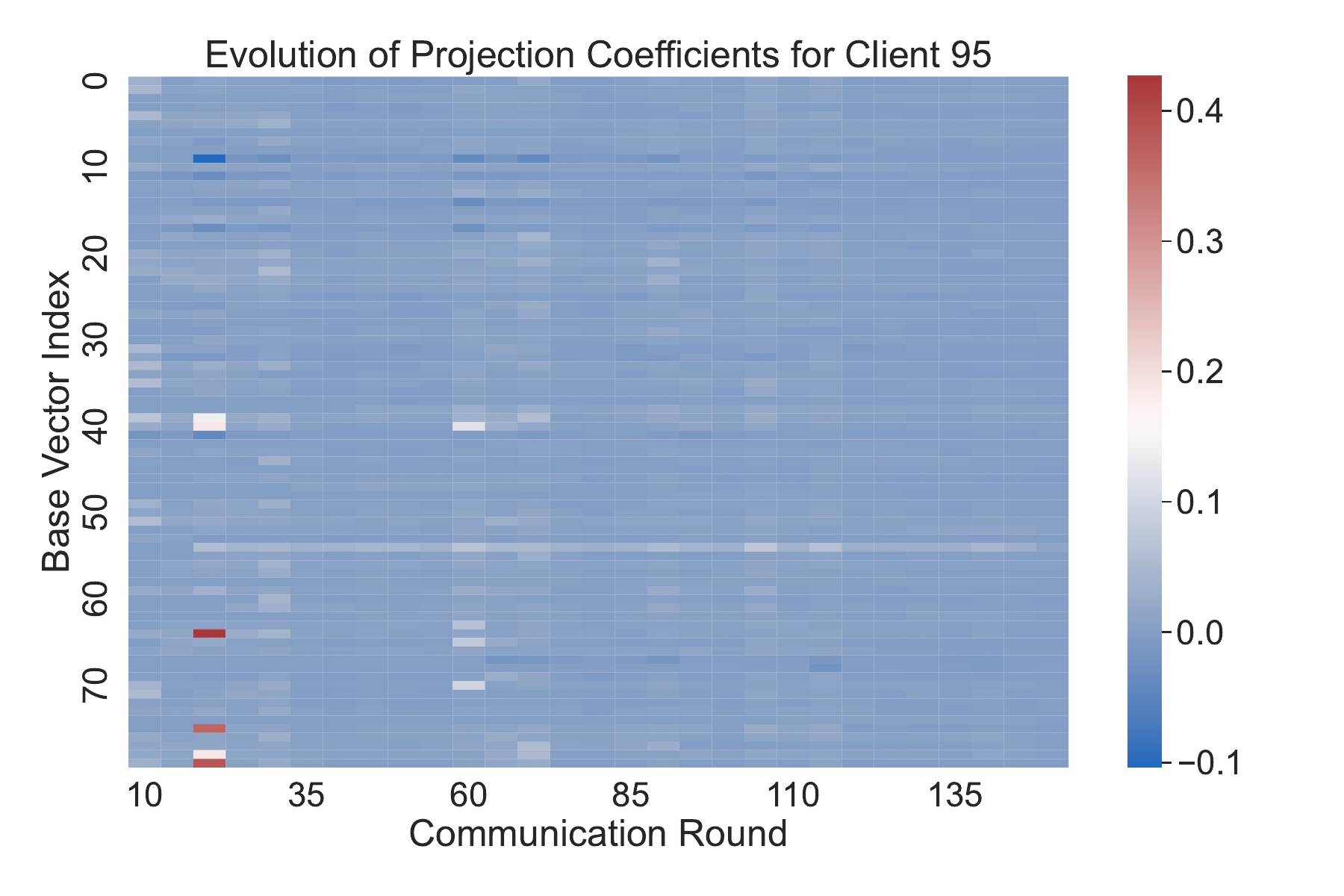}
        \caption{With regularization ($\lambda=0.5$), $|\mathcal{X}|=80$.}
        \label{subfig:reg-80}
    \end{subfigure}

    %----------- MAIN CAPTION FOR THE ENTIRE FIGURE -----------
    \caption{Comparison of projection coefficient evolution across different experimental settings. Each triplet of heatmaps shows results for clients 5, 55, and 95, respectively.}
    \label{fig:main-comparison}
\end{figure}

\end{document}

%% file: sections/intro.tex
\section{Introduction}
Federated Learning (FL) has become a popular distributed learning paradigm that enables model training on distributed data while \haoran{ensuring the data remains on client devices}
%\carlee{This is debatable, so it may be better to say keeping data local to the clients}
\citep{mcmahan2017communication,hard2018federated,shah2020training}. In a typical FL system, a central server coordinates training by aggregating model updates (e.g., gradients or weights) that clients compute on their local data.

In this work, we focus on FL applications where clients are edge devices, such as smartphones or Internet-of-Things (IoT) devices. These devices are often constrained by limited battery life, computational power, and network bandwidth, making it feasible for only a fraction of clients to participate in each training round \citep{ruan2020towards,10233897}, a setting known as \textit{partial client participation}. Compounding this challenge is the statistical heterogeneity of client data; since the data distributions are often non-IID across clients, their local model updates can vary significantly \citep{cho2020client}. When combined with partial client participation, this heterogeneity means the aggregated update in any given round may not be representative of the true global data distribution, leading to high variance and instability in the training process \citep{chen2020optimal,rizk2021optimal}. 
Furthermore, participation may be systematically biased, as clients with greater computational or network resources may  \rachid{participate} 
%be selected `
more frequently, introducing a convergence bias \citep{rodio2024fedstale}. While this bias can be mitigated by techniques like importance sampling \citep{chen2020optimal}--which re-weights updates to ensure all clients contribute equally in expectation---such methods, despite ensuring unbiased convergence, often introduce even higher variance into the training process \citep{zhang2025towards,wang2024delta}.

To mitigate the training instability caused by partial client participation, several prior works reuse stale information to incorporate updates from inactive clients \citep{gu2021fast,jhunjhunwala2022fedvarp,karimireddy2020scaffold,rodio2024fedstale}.
%\xutong{References here.}. 
These methods can be broadly categorized by how they leverage this stale information. 
\textit{(1) Direct Reuse of Stale Updates:} One line of work directly reuses the last known update from an inactive client as a surrogate for its current one. MIFA \citep{gu2021fast} and FedVARP \citep{jhunjhunwala2022fedvarp} both maintain the most recent update received from each client at the server. 
In each round, this memory is updated with fresh updates from participating clients, while the stale updates for non-participating clients are retained. 
MIFA then updates the global model by averaging all stored updates (fresh and stale), while FedVARP employs a SAGA-like variance reduction scheme that uses the stale updates as control variates \citep{jhunjhunwala2022fedvarp,defazio2014saga}. 
Similarly, SCAFFOLD \citep{karimireddy2020scaffold} corrects for client-drift using stateful client-specific control variates. When a client is inactive, its stale control variate is implicitly reused during the server-side aggregation, thus incorporating its past state into the global update.
ConFREE~\citep{zheng2025confree} mitigates gradient conflicts in personalized FL by projecting active clients’ updates to remove opposing components, but it does not address the misalignment of stale gradients arising from inactive clients.
\textit{(2) Weighted Stale Updates:} 
Acknowledging that highly stale information can be detrimental, another line of work proposes down-weighting stale updates \haoran{when aggregating them with fresh ones.}
%\carlee{staler? i.e., more weight is given to more recent updates} 
%\haoran{fedstale only use $\beta$ }
FedStale \citep{rodio2024fedstale} introduces a hyperparameter $\beta$ to form a convex combination of an update using only fresh information (akin to FedAvg) and one using both fresh and stale information (akin to FedVARP). 
This allows the algorithm to control the influence of stale updates. 
This strategy is also common in asynchronous FL, where the updates from slower devices are reweighted to ensure training stability \citep{xie2019asynchronous,miao2023robust,ma2024fedstaleweight}. 
\rachid{Although these methods can handle the variance  from partial participation, their main limitation is that they treat} stale gradients as fixed vectors.
\rachid{For example, for less active clients,} these approaches ignore the fact that the direction of their gradients becomes increasingly misaligned as the global model evolves. \rachid{This misalignment can destabilize the training process.}  
Crucially, these approaches lack a mechanism to actively correct the stale gradient's direction by leveraging the more current information provided by active clients.

\haoran{To address this \rachid{issue},
we propose FedSteer, a novel corrective mechanism that replaces a client's stale gradient with a more accurate estimate derived from an evolving gradient subspace.
FedSteer constructs this subspace from a cache of recent gradients provided by a small core set of clients $\mathcal{X}$.
\rachid{Although} each gradient vector exists in the high-dimensional space $\mathbb{R}^d$, the subspace's dimensionality is no more than the size of the core set $|\mathcal{X}|$.
When a client (indexed by $i$) is active, we project its gradient onto the current subspace to find low-dimensional coordinates ($\mathbf{s}_i$). 
The coordinates $\mathbf{s}_i$ can be interpreted as stable similarity coefficients, which capture the relationship between client $i$ and clients in the core set.  
% \rachid{The main idea behind our approach is that a given  client with an outdated gradient can be represented by a set of clients from the core group; by  merging them, we approximate that client’s dataset  distribution, yielding a more relevant gradient for that client for training.}
% FedSteer reuses these low-dimensional coordinates and stale gradients from the core set. 
When a client becomes inactive, FedSteer reuses these coordinates to the newly evolved subspace drifted by active clients from the core set to reconstruct a corrected gradient estimate. This process effectively ``steers'' the client's outdated information to approximate its true local gradient for the current model.
}

\haoran{We complement FedSteer with a selective caching strategy that significantly reduces the server's memory overhead.
The server only stores the gradient vectors (each in $\mathbb{R}^d$) for the small core set $\mathcal{X}$, and the low-dimensional coordinate vectors are cached for the entire client population ($N$).
FedSteer's memory cost is therefore $O(|\mathcal{X}|d+N|\mathcal{X}|)$, whereas prior methods \citep{gu2021fast,jhunjhunwala2022fedvarp,rodio2024fedstale} that cache all stale gradients incur a much higher cost of $O(Nd)$.\footnote{\haoran{In experiments, we set $N=100$, $|\mathcal{X}|=10$, and the model dimension $d$ exceeds $1.6\times 10^6$. With these values, our method reduces the gradient-caching memory overhead by nearly an order of magnitude.}}
To our knowledge, FedSteer is the first approach that corrects the direction of a client's stale update by leveraging collective, timely information from other clients in FL. 
}

\textbf{Our Contributions: }
\begin{itemize}
    \item We propose \textbf{FedSteer}, a novel algorithm that introduces a directional correction mechanism for stale updates. \haoran{FedSteer reconstructs a corrected update gradient for inactive clients by applying a client's stable, cached projection coordinates to a dynamic gradient subspace that is drifted by other active clients.}
    %\carlee{This does not make sense. The stale information itself is not a projection}, 
    FedSteer effectively steers outdated gradients toward the current global objective. We prove that this corrective projection minimizes the global aggregation variance, providing theoretical support for its performance.

    \item We introduce a \textbf{selective client caching} strategy to make FedSteer memory-efficient. This method \haoran{iteratively} optimizes a small, representative core set of clients \haoran{during a warm-up phase}.
    The updated gradients from this core set are sufficient to construct an effective gradient subspace, significantly reducing the server's storage cost. 
    %\carlee{make clear that this core set changes over time}

    \item We provide a rigorous \textbf{convergence analysis}, proving that FedSteer achieves a tighter convergence upper bound compared to prior methods that reuse stale information without a directional correction mechanism.

    \item Our \textbf{extensive empirical evaluation} on multiple datasets \haoran{under different settings of data and system heterogeneity demonstrates that FedSteer significantly outperforms state-of-the-art baselines. It achieves an accuracy gain of at least 7.9\% over the next-best method and prevents training collapse in the most challenging scenarios.}
    %particularly in settings with high client participation heterogeneity and extreme gradient staleness. 
    %\carlee{Not clear what extreme staleness means. Is this due to very infrequent participation from all clients, or heterogeneity in participation rates across clients? Is this supported by the derived convergence guarantee?}
\end{itemize}

% \carlee{Include paper outline}
\haoran{The rest of the paper is organized as follows. In Section~\ref{sec:methods}, we formulate the problem of aggregation variance and introduce our proposed method, FedSteer, detailing its corrective projection mechanism and selective caching strategy. We then provide a convergence analysis in Section~\ref{sec:converge} and present extensive experimental results in Section~\ref{sec:experiment}. Finally, we conclude the paper in Section~\ref{sec:conclusion}.
Detailed proofs for the convergence analysis are provided in Supplementary Material~\ref{app:proof}, along with additional experimental settings and results in Supplementary Material~\ref{app:exp}.
% \carlee{No need to get into subsections here unless we need to pad the paper's length}
% In Section \ref{sec:converge}, we provide a convergence proof for our algorithm, and then detail our extensive experiments in Section \ref{sec:experiment}, demonstrating FedSteer's effectiveness. Finally, we conclude the paper in Section \ref{sec:conclusion}.
}

%% file: sections/method.tex
\begin{figure*}[t]
    \centering
    \includegraphics[width=0.75\linewidth]{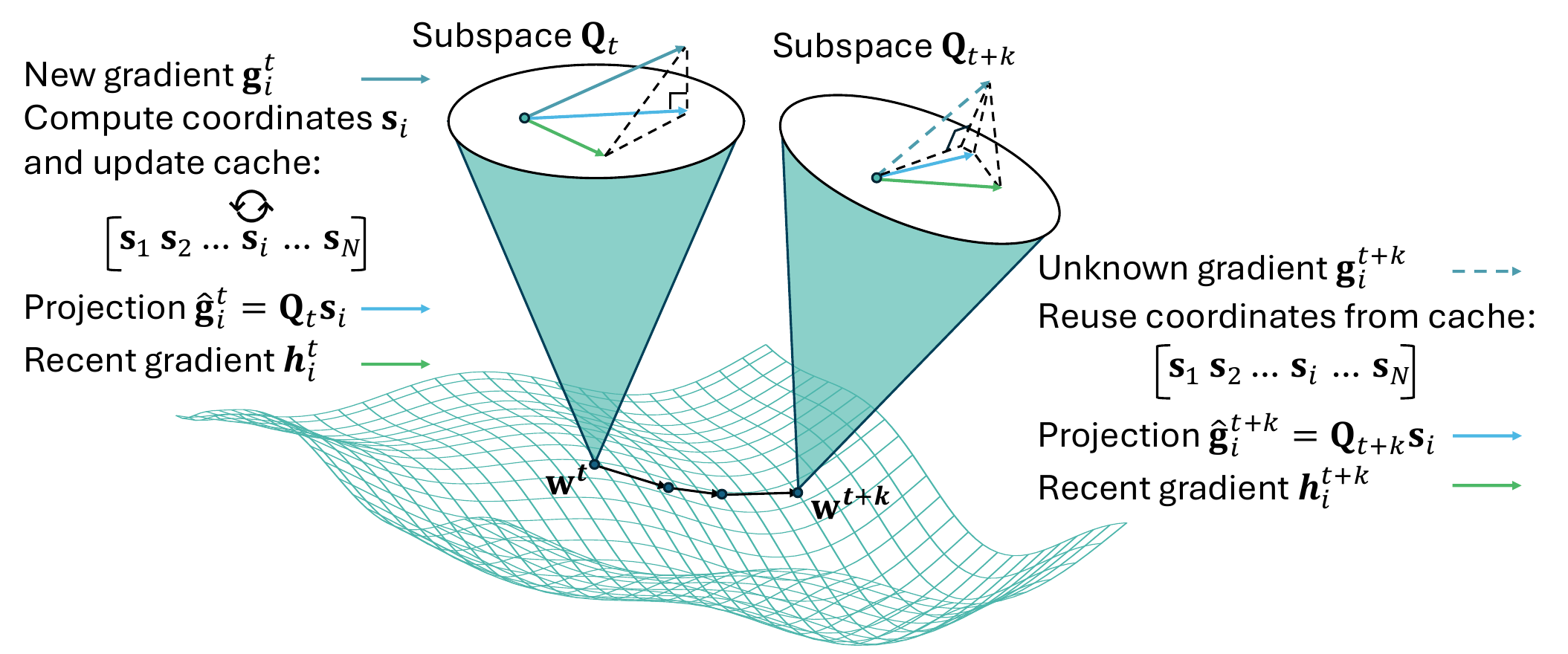}
    \caption{An overview of FedSteer's corrective mechanism. 
    At round $t$, an active client $i$'s gradient $\mathbf{g}_i^t$ is projected onto the subspace $\mathbf{Q}_t$ and the resulting low-dimensional coordinates $\mathbf{s}_i$ are cached. 
    In a subsequent round $t+k$, as long as client $i$ remains inactive, FedSteer reuses the cached coordinates $\mathbf{s}_i$ with the newly evolved subspace $\mathbf{Q}_{t+k}$ to form a corrected projection $\hat{\mathbf{g}}_i^{t+k}=\mathbf{Q}_{t+k}\mathbf{s}_i$. As illustrated, this ``steered'' projection can better estimate the true unknown gradient $\hat{\mathbf{g}}_{i}^{t+k}$ compared to its stale gradient $\mathbf{h}_i^{t+k}$.
    % we compute the projection of an active client $i$'s gradient ($\mathbf{g}_i^t$) onto the current gradient subspace $\mathbf{Q}_t$ and cache its coordinates $\mathbf{s}_i$. In a future round $t+k$, if client $i$ has been inactive, its last-known gradient is now stale. However, the subspace evolves to $\mathbf{Q}_{t+3}$, updated by other active clients to reflect the current optimization landscape. FedSteer corrects the stale information by applying the cached coordinates to the new subspace, forming a new projection $\hat{\mathbf{g}}_i^{t+3} = \mathbf{Q}_{t+3}\mathbf{s}_i$. This new projection is steered toward the current objective and serves as a better estimate of the client's true (but unknown) gradient $\mathbf{g}_i^{t+3}$. 
    %\carlee{what do the red Xes mean? Why do we need to show $q_i^t$? It is also not clear from the figure that we reuse the same projection coefficients}   
    }
        \vspace{-0.1in}
    \label{fig:overview}
\end{figure*}

\section{Methods}
\label{sec:methods}

\subsection{Problem Formulation}
\label{sec:problem}
We consider an FL system with a set of $N$ clients, indexed by $i\in\mathcal{N} = \{1, \dots, N\}$. The global objective is to minimize a weighted average of the clients' local loss functions:
\begin{equation}
\min_{\mathbf{w} \in \mathbb{R}^d} F(\mathbf{w}) \triangleq \sum_{i=1}^{N} d_i F_i(\mathbf{w}),
\end{equation}
where $\mathbf{w} \in \mathbb{R}^d$ is the global model and $F_i(\mathbf{w}) \triangleq \mathbb{E}_{\xi_i \sim \mathcal{D}_i}[\ell(\mathbf{w}, \xi_i)]$ is the local objective for client $i$, representing the expected loss over its data distribution $\mathcal{D}_i$. Each client's contribution is scaled by a weight $d_i$, typically set in proportion to its dataset size, such that $d_i = \frac{|\mathcal{D}i|}{\sum{j\in\mathcal{N}} |\mathcal{D}_j|}$.
\haoran{We model each client's participation as an independent Bernoulli trial with a client-specific participation probability $p_i$.}

\textbf{FL training process: }
In each global round $t=1,2,\dots,T$, 
a subset of clients $\mathcal{A}_t$ is active, with client $i$ participating with probability $p_i$. Active clients perform $E$ local training epochs (e.g., using SGD) starting from $\mathbf{w}_i^{0}=\mathbf{w}^t$. After training, each active client $i \in \mathcal{A}_t$ computes its total update $\mathbf{g}_i^t=\mathbf{w}_i^{0}-\mathbf{w}_i^E$ and sends it to the server. 
Define the stale gradient of client $i$ as: 
\(
\text{if } i \in \mathcal{A}_{t}\text{: }
\mathbf{h}_i^{t+1}=
        \mathbf{g}_{i}^{t},   \text{ otherwise: }
        \mathbf{h}_i^{t+1}=\mathbf{h}_{i}^{t}.
\)

\subsection{Corrective Projections for Stale Gradients}
\label{sec:projection}
To address the issue of extreme staleness, we propose \textbf{FedSteer}, a method that corrects a stale client update by projecting the new gradient onto a dynamically evolving gradient subspace, and reuse the subspace coordinates instead. As shown in Fig. \ref{fig:overview}, this process steers the outdated information toward the current optimization landscape, making it relevant for the global update.

\subsubsection{The Dynamic Gradient Subspace}
The foundation of our method is a low-dimensional gradient subspace constructed from recent client updates. The server maintains a cache of the most recent raw gradients, denoted as $\mathbf{h}_i^t \in \mathbb{R}^d$, from a small, representative subset of clients $\mathcal{X} \subseteq \mathcal{N}$, referred to as the \textbf{core set} \haoran{(the selection process is described in Section \ref{sec:caching})}.
% \carlee{How is this coreset chosen?}
To ensure these gradients can form a well-defined basis, we make the following assumption:
\begin{assumption}[Non-trivial gradients]
\label{assumption:scale}
For any cached gradient $\mathbf{h}_i^t, i\in\mathcal{X}$, its magnitude is strictly positive, i.e., $\|\mathbf{h}_i^t\|_2>0$. 
\end{assumption}
The assumption allows us to safely normalize each raw gradient, improving numerical stability. 
The normalized basis vector $\bar{\mathbf{h}}_i^t$ is defined as:
\begin{align}
\bar{\mathbf{h}}_i^t=\frac{\mathbf{h}_i^t}{\|\mathbf{h}_i^t\|_2}.
\end{align}

The dynamic subspace is then formally defined by a matrix $\mathbf{Q}_t \in \mathbb{R}^{d \times |\mathcal{X}|}$, whose columns are these normalized basis vectors:
\begin{equation}
    \mathbf{Q}_t = \begin{bmatrix} 
    \bar{\mathbf{h}}_{j_1}^t & \dots &\bar{\mathbf{h}}_{j_k}^t &\dots & \bar{\mathbf{h}}_{j_{|\mathcal{X}|}}^t 
    \end{bmatrix}_{j_k \in \mathcal{X}}\label{eq:Q}.
\end{equation}
While the core set $\mathcal{X}$ can be fixed, the subspace itself is dynamic, as the basis vectors $\bar{\mathbf{h}}_j^t$ are updated whenever a client $j \in \mathcal{X}$ participates in a training round. 
\haoran{Even if some core set clients are temporarily inactive, the subspace $\mathbf{Q}_t$ as a whole continues to evolve through the participation of its other members. Consequently, $\mathbf{Q}_t$ serves as a continuously updated, low-dimensional representation of the dominant directions of recent client update gradients.
This set of dominant directions provides a basis for estimating any client's update. Since $\mathcal{X}$ is chosen to be representative, projecting a client's gradient onto $\mathbf{Q}_t$ effectively decomposes its update into a combination of these core directions.
The resulting coordinates, which capture the client's similarity to the core set $\mathcal{X}$, are then reused to reconstruct an estimated gradient for any inactive client.
}

% \carlee{Can't these become stale? It's not quite clear what you mean by optimization landscape: this subspace covers the space of different clients' gradients, but that's not necessarily related to the global objective's structure around the current model}

\subsubsection{Corrective Projection via Cached Coordinates}
Rather than directly reusing a stale gradient $\mathbf{h}_i^t$, 
% \carlee{This is rather confusing since you use $\mathbf{q}$ to denote the directions of the gradient subspace too}
which risks being outdated and destabilizing the training process, FedSteer computes a more relevant approximation, $\hat{\mathbf{g}}_i^t$, that lies within the dynamic subspace $\mathbf{Q}_t$. This approximation is defined as a linear combination of the subspace's basis vectors:
\begin{equation}
\hat{\mathbf{g}}_i^t = \mathbf{Q}_t \mathbf{s}_i.
\end{equation}
Here, the vector $\mathbf{s}_i \in \mathbb{R}^{|\mathcal{X}|}$ denotes the coordinates of the gradient approximation within the subspace. The optimal coordinates are those that make $\hat{\mathbf{g}}_i^t$ the best possible estimate of the true gradient $\mathbf{g}_i^t$. A standard projection, however, can yield coordinates with large magnitudes that overfit to the specific basis vectors of $\mathbf{Q}_t$ at a single time step. To find a more stable and generalizable coordinate representation, we introduce a regularization term into the optimization.

\begin{algorithm}[t]%[H]
	\caption{FedSteer algorithm}
 \label{algo:fedsteer}
 \small
	\begin{algorithmic}[1]
        \State \textbf{Input:} clients: $i\in\mathcal{N}=\{1,2,\dots,N\}$, client participation distribution $\{p_i\}_{i\in\mathcal{N}}$, core set $\mathcal{X}$, regularization factor $\lambda$, learning rate $\eta_c$ and $\eta_s$. 
        \State \textbf{Objective:} $\min \sum_{i\in \mathcal{N}} d_{i} F_{i}(\mathbf{w})$
        \State \textbf{Initialization:} The global model weights $\mathbf{w}^1$, $\mathbf{s}_i^{\mathrm{cache}}=\mathbf{0}$ for all clients.
        %randomly initialize the set of active clients for each task: $\mathcal{A}_1(s)$ 
        \For{global round $t=1,\cdots, T$}
        \State Determine the set of active clients $\mathcal{A}_t$ from $\{p_i\}_{i\in\mathcal{N}}$
        \For{each client $i \in\mathcal{A}_t $, in parallel}
        \State Local initialization $\mathbf{w}_i^0=\mathbf{w}^t$
        \State SGD for $E$ epochs, obtain $\mathbf{g}_i^t=\mathbf{w}_i^0-\mathbf{w}_i^E$
        \State Send $\mathbf{g}_i^t$ to the server
        \EndFor
        % \State Update stale update $q_i^t$
        \State At server:
        \State Receive updates $\mathbf{g}_i^t$ from active clients
        \State Update gradient estimation $\hat{\mathbf{g}}_i^t=\mathbf{Q}_t\mathbf{s}_i^{\mathrm{cache}}$
        \State Server aggregation: $\mathbf{w}^{t+1}=\mathbf{w}^t-\eta_{s}\Delta_t$
        \State Compute $\mathbf{s}_i^*$ for $i\in\mathcal{A}_t$ (Eq. \ref{eq:solution}), cache $\mathbf{s}_i^{\mathrm{cache}}=\mathbf{s}_i^*$
        \State Update $\mathbf{h}_i^{t+1}$ and $\mathbf{Q}_{t+1}$ (Eq. \eqref{eq:Q})
        \State Broadcast $\mathbf{w}^{t+1}$ to clients
        \EndFor
	\end{algorithmic} 
\end{algorithm}

\textbf{For an active client $i \in \mathcal{A}_t$}, the server computes its optimal coordinates $\mathbf{s}_i^*$ by solving the following regularized least-squares (Ridge Regression) problem:
\begin{equation}
\min_{\mathbf{s}_i \in \mathbb{R}^{|\mathcal{X}|}} \| \mathbf{g}_i^t - \mathbf{Q}_t \mathbf{s}_i \|^2 + \lambda \|\mathbf{s}_i\|^2, \forall\; i\in\mathcal{A}_t,
\label{eq:problem}
\end{equation}
where the regularization parameter $\lambda > 0$ penalizes large coefficient values. This prevents the projection from relying too heavily on any single cached gradient in the subspace, promoting a more robust representation. The closed-form solution is:
\begin{equation}
\mathbf{s}_i^* = (\mathbf{Q}_t^\top \mathbf{Q}_t + \lambda \mathbf{I})^{-1} \mathbf{Q}_t^\top \mathbf{g}_i^t.
\label{eq:solution}
\end{equation}
\begin{remark}
The matrix $(\mathbf{Q}_t^\top \mathbf{Q}_t + \lambda \mathbf{I})$ is guaranteed to be positive definite, and therefore invertible, for any $\lambda > 0$. This ensures the solution in Eq.~\eqref{eq:solution} is always unique and well-defined, regardless of whether the columns of $\mathbf{Q}_t$ are linearly independent.
\end{remark}
Once computed, $\mathbf{s}_i^*$ are stored in the server's cache for client $i$: $\mathbf{s}_i^{\mathrm{cache}}=\mathbf{s}_i^*$, and its gradient estimate is written as
$
\hat{\mathbf{g}}_i^t = \mathbf{Q}_t \mathbf{s}_i^{\mathrm{cache}}
$, which is used in global aggregation. 

\textbf{For an inactive client $i \notin \mathcal{A}_t$}, its true gradient is unknown.
\haoran{FedSteer avoids directly reusing} its stale gradient $\mathbf{h}_i^t$.
Instead, it retrieves the client's cached coordinates $\mathbf{s}_i^{\mathrm{cache}}$, 
\haoran{which encode the stable relationship between client $i$'s update and the core set's basis vectors.}
It then applies these coordinates to the \textit{current} subspace $\mathbf{Q}_t$.
%\haoran{which is drifted by active clients in the core set and reflects the most recent optimization directions.}
As illustrated in Fig. \ref{fig:overview}, 
\haoran{because this subspace has been updated by active clients, the resulting reconstruction $\hat{\mathbf{g}}_i^t=\mathbf{Q}_t \mathbf{s}_i^{\mathrm{cache}}$ effectively ``steers'' the inactive client's historical information to align with the current optimization path.}

% These optimal coefficients are found by solving the least-squares problem:
% \begin{equation}
%     \min_{\mathbf{s}_i \in \mathbb{R}^{|\mathcal{X}|}} \| G_i^t - Q_t \mathbf{s}_i \|_2^2\, ,\; \forall i\in\mathcal{N}
% \end{equation}
% The closed-form solution for the coefficient vector $\mathbf{s}_i$ is given by:
% \begin{equation}
%     \mathbf{s}_i = (Q_t^\top Q_t)^{\dagger} Q_t^\top G_i^t
% \end{equation}
% where $(\cdot)^{\dagger}$ denotes the Moore-Penrose pseudoinverse. The vector $\mathbf{s}_i$ represents the coordinates of the projection in the basis of cached gradients.

\subsubsection{The FedSteer Aggregation Rule}
The final server update uses the projected gradients $\hat{\mathbf{g}}_i^t=\mathbf{Q}_t \mathbf{s}_i^{\mathrm{cache}}$ as a baseline estimate for all clients and applies a variance reduction correction using the true gradients from the active set:
\begin{align}
    \mathbf{w}_{t+1} &= \mathbf{w}_t - \eta_s \Delta_t \\
    \Delta_t &= \underbrace{\sum_{i\in\mathcal{N}} d_i \hat{\mathbf{g}}_i^t}_{\text{Projection Baseline}} + \underbrace{\sum_{i\in\mathcal{A}_t} \frac{d_i}{p_{i}} (\mathbf{g}_i^t - \hat{\mathbf{g}}_i^t)}_{\text{Residual Correction}}\label{eq:aggregation}
\end{align}
\haoran{
\begin{proposition}[Unbiased Estimator]\label{pro:1}
The FedSteer global update $\Delta_t$ is an unbiased estimator of the true global gradient update: 
\(
\mathbb{E}[\Delta_t]=\sum_{i\in\mathcal{N}} d_i \mathbf{g}_i^t
\).
\end{proposition}
\begin{proof}
Taking the expectation over client sampling ($\mathbb{E}[\mathbbm{1}_{i\in\mathcal{A}_t}] = p_i$):
\begin{align*}
\mathbb{E}[\Delta_t] = \sum_{i\in\mathcal{N}}d_{i}\hat{\mathbf{g}}_{i}^{t} + \mathbb{E}\left[\sum_{i\in\mathcal{N}}\mathbbm{1}_{i\in\mathcal{A}_t}\frac{d_{i}(\mathbf{g}_{i}^{t}-\hat{\mathbf{g}}_{i}^{t})}{p_{i}}\right]
= \sum_{i\in\mathcal{N}}d_{i}\mathbf{g}_{i}^{t}
\end{align*}
The estimator is therefore unbiased.
\end{proof}
The unbiased estimator in Eq. \eqref{eq:aggregation} reduces variance by  leveraging a baseline estimate ($\hat{\mathbf{g}}_i^t$) which allows the active clients to estimate only the small, low-variance residual error ($\mathbf{g}_i^t-\hat{\mathbf{g}}_i^t$). 
}
% This aggregation rule is unbiased and reduces the variance caused by client sampling. 
%\carlee{How does it reduce the variance? We should state the unbiasedness result as a formal theorem/proposition} 
The pseudocode of FedSteer is provided in Algorithm \ref{algo:fedsteer}.

\begin{remark}[Reduction to prior work]
\label{remark:reduction}
FedSteer reduces to prior methods by selecting specific coordinates $\{\mathbf{s}_i\}_{i\in\mathcal{N}}$. Assuming $\mathcal{X}=\mathcal{N}$, FedSteer reduces to: 
(1) FedVARP \citep{jhunjhunwala2022fedvarp}, by setting coordinates to $\mathbf{s}_i^{\mathrm{cache}} = \|\mathbf{h}_i^t\|_2 \cdot \mathbf{e}_i$, where $\mathbf{e}_i$ is a one-hot vector, to recover the exact gradient estimate $\hat{\mathbf{g}}_i^t = \mathbf{h}_i^t$; and (2) FedStale \citep{rodio2024fedstale}, by choosing coordinates $\mathbf{s}_i^{\mathrm{cache}} = \beta \|\mathbf{h}_i^t\|_2 \cdot \mathbf{e}_i$ to yield the down-weighted estimate $\hat{\mathbf{g}}_i^t = \beta \mathbf{h}_i^t$.
\end{remark}

\textbf{Complexity analysis:}
FedSteer incurs no additional communication cost over standard FL (e.g., FedAvg/FedVARP): only active clients transmit updates, and client-side computation remains unchanged. The server solves a small ridge regression with complexity $O(k^2 d + k^3)$, which is modest since $k \ll N$ and comparable to standard aggregation. FedSteer is also more memory-efficient than stale-gradient methods (FedVARP/FedStale/MIFA), reducing storage from $O(Nd)$ to $O(Nk + kd)$ by keeping low-dimensional coordinates for all clients and full gradients only for a small core set. Detailed complexity analysis is provided in the Supplementary Material~\ref{app:communication}.

\subsection{Variance Minimization via Projection}
\label{sec:ProjectionEqualToMinVariance}
We show that having each client $i$ minimize its local projection error (Eq.~\eqref{eq:problem} with $\lambda=0$) over its own variable vector $\mathbf{s}_i$ is equivalent to minimizing the variance of the global update $\Delta_t$ over the collective set of variables $\{\mathbf{s}_i\}_{i\in\mathcal{N}}$.
%This equivalence is formalized in Theorem \ref{them:variance}. 

\begin{theorem}[Variance minimization]
\label{them:variance}
Given the true gradients $\{\mathbf{g}_i^t\}_{i \in \mathcal{N}}$ and the subspace matrix $\mathbf{Q}_t$, the variance of the global update, $\mathrm{Var}(\Delta_t)$, is minimized when each client's coordinate vector $\mathbf{s}_i$ is the solution to the following least-squares problem (Eq. \eqref{eq:problem} with $\lambda=0$) \haoran{for each client $i\in\mathcal{N}$}:
% \carlee{why is this ``independent''?}
\begin{equation}
\mathbf{s}_i = \underset{\mathbf{s}_i' \in \mathbb{R}^{|\mathcal{X}|}}{\mathrm{argmin}} \; \| \mathbf{g}_i^t - \mathbf{Q}_t \mathbf{s}_i' \|^2.
\end{equation}
\end{theorem}
\begin{proof}
% The global update is $\Delta_t = \sum_{i\in\mathcal{N}} d_i \hat{\mathbf{g}}_i^t + \sum_{i\in\mathcal{A}_t} \frac{d_i}{p_{i}} (\mathbf{g}_i^t - \hat{\mathbf{g}}_i^t)$. 
The variance is taken over the random sampling of the active client set $\mathcal{A}_t$. Since the first term $\sum_{i\in\mathcal{N}} d_i \hat{\mathbf{g}}_i^t$ is constant with respect to this sampling, it does not contribute to the variance. The variance of the update is therefore:
\(
        \mathrm{Var}(\Delta_t) 
    = \mathrm{Var}\left(\sum_{i\in\mathcal{N}} \mathbbm{1}_{i \in \mathcal{A}_t} \frac{d_i}{p_i} (\mathbf{g}_i^t - \hat{\mathbf{g}}_i^t)\right)
,
\)
where $\mathbbm{1}_{i \in \mathcal{A}_t}$ is the indicator variable defined in Proposition \ref{pro:1}. Since clients are sampled independently, the variance of the sum is the sum of the variances:
\(
\mathrm{Var}(\Delta_t) = \sum_{i\in\mathcal{N}} \mathrm{Var}\left(\mathbbm{1}_{i \in \mathcal{A}_t} \frac{d_i}{p_i} (\mathbf{g}_i^t - \hat{\mathbf{g}}_i^t)\right)
\).
For each client $i$, we use $\mathrm{Var}(\mathbf{X}) = \mathbb{E}[\|\mathbf{X}\|^2] - \|\mathbb{E}[\mathbf{X}]\|^2$. 
Let $\mathbf{r}_i = \mathbf{g}_i^t - \hat{\mathbf{g}}_i^t$ denote the residual vector. The variance for client $i$'s term is:
\begin{align*}
    &\mathbb{E}\left[\left\|\mathbbm{1}_{i \in \mathcal{A}_t} \frac{d_i}{p_i} \mathbf{r}_i\right\|^2\right] - \left\|\mathbb{E}\left[\mathbbm{1}_{i \in \mathcal{A}_t} \frac{d_i}{p_i} \mathbf{r}_i\right]\right\|^2 \\
    % &= \mathbb{E}[\mathbbm{1}_{i \in \mathcal{A}_t}] \left(\frac{d_i}{p_i}\right)^2 \|\mathbf{r}_i\|^2 - \left\|\mathbb{E}[\mathbbm{1}_{i \in \mathcal{A}_t}] \frac{d_i}{p_i} \mathbf{r}_i\right\|^2 \\
    &= \frac{d_i^2(1-p_i)}{p_i} \|\mathbf{g}_i^t - \hat{\mathbf{g}}_i^t\|^2.
\end{align*}
Therefore, minimizing the variance is equivalent to:
\[
\min_{\{\mathbf{s}_i\}_{i \in \mathcal{N}}} \sum_{i\in\mathcal{N}} \frac{d_i^2(1-p_i)}{p_i} \|\mathbf{g}_i^t - \mathbf{Q}_t \mathbf{s}_i\|^2.
\]
Since the objective is separable across clients, i.e., it consists of a sum of independent terms each involving only a single variable $\mathbf{s}_i$, the global optimization problem decouples into independent weighted least-squares problems for each client.
\end{proof}

\begin{algorithm}[t]
	\caption{Selective client caching}
	\label{algo:caching}
	\small
	\begin{algorithmic}[1]
	    \State \textbf{Input:} Full client set $\mathcal{N}$, core set size $k\leq |\mathcal{N}|$, number of selection cycles $T_0$,
        number of swap iterations $I_{max}$ per cycle.
		\State \textbf{Initialization:} Randomly initialize a core set $\mathcal{X} \subseteq \mathcal{N}$ of size $k$.
        \For{selection cycle $t=1,\dots,T_0$}
        \State Collect $\mathbf{g}_i^t$ from all clients $i\in\mathcal{N}$ \haoran{asynchronously}
		\For{$iter = 1, \dots, I_{max}$}
        % \State Feasible swap-in $j_{in}\in\mathcal{N}\setminus\mathcal{X}$ and swap-out $j_{out}\in\mathcal{X}$
		    \State Find the optimal swap $(j_{out}^*, j_{in}^*)$ by solving:
    		    \State $\underset{j_{out} \in \mathcal{X}, \, j_{in} \in \mathcal{N} \setminus \mathcal{X}}{\mathrm{argmin}} \; J\big((\mathcal{X} \setminus \{j_{out}\}) \cup \{j_{in}\}\big)$
    		    \State Let $\mathcal{X}_{\text{new}} = (\mathcal{X} \setminus \{j_{out}^*\}) \cup \{j_{in}^*\}$
    		    \If{$J(\mathcal{X}_{\text{new}}) < J(\mathcal{X})$}
    		        \State $\mathcal{X} = \mathcal{X}_{\text{new}}$ \Comment{Update $\mathcal{X}$ with the best swap}
    		    \Else
    		        \State \textbf{break} \Comment{No further improvement}
    		    \EndIf
		\EndFor
        \State Aggregation: $\mathbf{w}^{t+1}=\mathbf{w}^t-\eta_s \sum_{i\in\mathcal{N}} d_i \mathbf{g}_i^t$
		\EndFor
        \State \textbf{return} the optimized core set $\mathcal{X}$
	\end{algorithmic}
\end{algorithm}

\subsection{Selective Client Caching}
\label{sec:caching}
We propose a method to select an effective \textbf{core set} $\mathcal{X}$ during an initial warm-up phase.
After the warm-up phase, the optimized core set $\mathcal{X}$ is fixed for the formal training.\footnote{To ensure a fair comparison with baselines, the global model's weights are re-initialized before the formal training process begins.}
\haoran{
The warm-up phase consists of $T_0$ selection cycles.
In each selection cycle, the server holds the global model constant and asynchronously collects available gradient from each client as they naturally become active. 
Once gradients from all $N$ clients have been gathered, the server has a complete snapshot of the gradient landscape and proceeds with the optimization step for $\mathcal{X}$.
It then performs an aggregation including all clients to update the global model for the next selection cycle.
}

\haoran{The core set selection within a selection cycle is performed via} a \textbf{greedy local search} that iteratively improves an initially random subset $\mathcal{X}$. 
The objective is to find a core set whose gradient subspace can accurately reconstruct the gradients of the entire client population. Specifically, the algorithm seeks to find a set $\mathcal{X}$ that minimizes the total regularized projection error, $J(\mathcal{X})$, defined as:
\begin{equation}
J(\mathcal{X})=\sum_{i\in\mathcal{N}}d_i \left( \| \mathbf{G}_i^t - \mathbf{Q}_t(\mathcal{X}) \mathbf{s}_i \|_2^2 + \lambda \|\mathbf{s}_i\|_2^2 \right),
\end{equation}
where $\mathbf{Q}_t(\mathcal{X})$ is the subspace matrix generated from the gradients of clients in $\mathcal{X}$, following the same definition in Eq. \eqref{eq:Q}.
Since an exhaustive search for the optimal $\mathcal{X}$ is computationally intractable, we refine the set through \textbf{single-client swaps}. In each iteration, it evaluates exchanging one client in the core set with one outside of it and executes the single swap that yields the greatest reduction in the objective $J(\mathcal{X})$. 
To manage the computational cost, we limit this search to a small number of iterations (e.g., at most five) per \haoran{selection cycle}. 
\haoran{
For each selection cycle, by updating the model and re-optimizing the core set, we ensure $\mathcal{X}$ is robust because its selection is based on performance across several different versions of the model. This prevents the choice of $\mathcal{X}$ from overfitting to a specific model state, ensuring $\mathcal{X}$ remains representative as the model continues to train.
The pseudocode of the algorithm is provided in Algorithm \ref{algo:caching}.
}

\textbf{Complexity analysis:}
The selective caching algorithm has complexity $O(T_{0} I_{\max} N k)$, where $T_{0}$ is the number of selection cycles and $I_{\max}$ is the number of swap iterations per cycle. This cost is incurred only during a short warm-up phase on the server (typically the first $T_{0}=5$ rounds), after which the core set is fixed and no further selection overhead is introduced for the remaining training rounds. To scale to very large client populations $N$, we also employ a subsampled selection strategy, which achieves comparable performance (Supplementary Material~\ref{app:subsampling}).

\section{Convergence Analysis}
\label{sec:converge}
We make the following standard assumptions. 

\begin{assumption}[$L$-smoothness]\label{assumption:L}
Each $F_{i}$ is L-smooth, and thus $F=\sum_{i\in \mathcal{N}} d_{i} F_{i}$ is also L-smooth.  
\end{assumption}

\begin{assumption}[Bounded variance at client-level]\label{assumption:B}
The stochastic gradient at each client is an unbiased estimator of the local gradient: $\mathbb{E}_{\xi_i\sim \mathcal{D}_i}[\nabla F_i(\mathbf{w},\xi_i)]=\nabla F_i(\mathbf{w})$, and its variance is bounded: $Var_{\xi_i\sim\mathcal{D}_i}(\nabla F_i(\mathbf{w},\xi_i))\leq \sigma$.
\end{assumption}

\begin{assumption}[Bounded variance across clients]
\label{assum:bounded_variance}
There exists a constant $\sigma_g^2 > 0$ such that the difference between the local gradient at the $i$-th client and the global gradient is bounded, that is
$$ \Vert \nabla F_i(\mathbf{w}) - \nabla F(\mathbf{w}) \Vert^2 \le \sigma_g^2, \quad \forall \mathbf{w}, i. $$
\end{assumption}

\begin{assumption}[Partial and heterogeneous client participation]
\label{assum:participation}
In each round $t$, client $i$ participates with a probability $p_i$, independently of previous rounds and other clients. $p_i$ is bounded: $p_{min}<p_i<p_{max}$. 
\end{assumption}

\begin{theorem}[Convergence analysis]
Under Assumptions \ref{assumption:L}-\ref{assum:participation}, $\eta_c\leq \frac{1}{4\sqrt{2E(E-1)}}$, and $\eta_s\leq \frac{1}{2L}$, the iterates $\{\mathbf{w}^t\}$ generated by FedSteer satisfy: 
\begin{align*}
&\min_{t\in[1,T]} \mathbb{E}\|\nabla F(\mathbf{w}^t)\|^2
\leq
\underbrace{\frac{4(F(\mathbf{w}^1)-F(\mathbf{w}^*))}{\eta_s T}}_{\text{iterate initialization error}} \\
&+ 
\underbrace{4 \eta_s L \frac{\sum_{t=1}^T \mathbb{E}\|\Delta^t-\mathbb{E}[\Delta^t]\|^2}{T}}_{\text{partial update variance error}} \\
&+\underbrace{\Gamma\sigma^2[2\eta_c^2L^2(E-1)+\frac{1}{TE}]  }_{\text{stochastic gradient error}}\\
&+\underbrace{(2+\Gamma) 8\eta_c^2L^2E(E-1) \sigma_g^2}_{\text{error from data heterogeneity}} \\
&+\underbrace{\frac{1}{T} \frac{\Gamma(1-p_{min})}{p_{avg}} \frac{1}{N} \sum_{i=1}^N \|\nabla F_i(\mathbf{w}^1)-\mathbf{h}_i^1 \|^2}_{\text{memory initialization error}},
\end{align*} 
where $\Gamma=\frac{p_{min} p_{avg}}{\eta_s L(2-p_{min})}$ with $p_{avg} = (1/N) \sum_{i=1}^N p_i$. 
\label{them:converge}
\end{theorem}
We provide the proof of Theorem \ref{them:converge} in the Supplementary Material \ref{app:proof}. 
The convergence bound consists of several terms.
Some, like the \textit{iterate initialization error}, vanish as the number of global rounds $T$ grows.
Others, including the \textit{partial update variance error}, \textit{stochastic gradient error}, and \textit{error from data heterogeneity} contribute to a non-vanishing error floor, which is characteristic of stochastic, non-convex optimization in FL.
As stated in Theorem \ref{them:variance}, FedSteer can explicitly minimize the \textit{partial update variance} term $\mathbb{E}\left[\|\Delta^t - \mathbb{E}[\Delta^t]\|^2\right]$.
By reducing the magnitude of this dominant, non-vanishing term, FedSteer achieves a tighter convergence bound \haoran{compared to prior methods including FedVARP \citep{jhunjhunwala2022fedvarp} and FedStale \citep{rodio2024fedstale}}. 
% \carlee{Theorem 2 says nothing about how accurate the stationary point is (in fact it does not even say that a stationary point is reached). Also, tighter bound with respect to what? FedAvg? FedStale?}
As shown in Remark~\ref{remark:reduction}, methods like FedVARP and FedStale can be viewed as constrained versions of FedSteer where the coordinates $\mathbf{s}_i$ are restricted to be one-hot vectors.
This suboptimal choice precludes the minimization of the partial update variance. Consequently, these methods, with a larger variance in the global update, result in a looser theoretical convergence bound.
%and also lead to more erratic and unstable training dynamics.

%% file: sections/experiment.tex
\section{Experiments}
\label{sec:experiment}
\begin{table*}[t]
\centering
\caption{Final average test accuracy on EMNIST, Fashion-MNIST, and CIFAR-10 under moderate ($\gamma=0.7$) and extreme ($\gamma=0.9$) heterogeneity; best results are in bold. FedSteer consistently outperforms all baselines. Its advantage is most significant under extreme heterogeneity (EMNIST, $\gamma=0.9$). }
\label{tab:result}
\small
\begin{tabular}{@{}l|cc|cc|cc@{}}
\toprule
\multirow{2}{*}{Methods} 
& \multicolumn{2}{c|}{EMNIST} 
& \multicolumn{2}{c|}{Fashion-MNIST} 
& \multicolumn{2}{c}{CIFAR-10} \\ 
\cmidrule(l){2-7}
& $\gamma$=0.9 & $\gamma$=0.7 
& $\gamma$=0.9 & $\gamma$=0.7
& $\gamma$=0.9 & $\gamma$=0.7 \\ 
\midrule
Full participation       
& 0.747{\tiny$\pm$.023} & 0.747{\tiny$\pm$.023}              
& 0.686{\tiny$\pm$.006} & 0.686{\tiny$\pm$.005}             
& 0.695{\tiny$\pm$.004} & 0.695{\tiny$\pm$.004} \\

FedAvg \citep{mcmahan2017communication}                  
& 0.314{\tiny$\pm$.015} & 0.717{\tiny$\pm$.027}          
& 0.509{\tiny$\pm$.010} & 0.609{\tiny$\pm$.010}             
& 0.591{\tiny$\pm$.010} & 0.631{\tiny$\pm$.016} \\

FedVARP \citep{jhunjhunwala2022fedvarp}                  
& 0.281{\tiny$\pm$.014} & 0.731{\tiny$\pm$.026}          
& 0.486{\tiny$\pm$.008} & 0.588{\tiny$\pm$.007}             
& 0.569{\tiny$\pm$.007} & 0.604{\tiny$\pm$.013} \\

FedStale \citep{rodio2024fedstale}                
& 0.309{\tiny$\pm$.014} & 0.734{\tiny$\pm$.027}          
& 0.497{\tiny$\pm$.007} & 0.596{\tiny$\pm$.009}             
& 0.585{\tiny$\pm$.005} & 0.638{\tiny$\pm$.009} \\

MIFA \citep{gu2021fast}                    
& 0.243{\tiny$\pm$.005} & 0.287{\tiny$\pm$.006}          
& 0.461{\tiny$\pm$.011} & 0.500{\tiny$\pm$.013}             
& 0.530{\tiny$\pm$.008} & 0.551{\tiny$\pm$.012} \\

SCAFFOLD \citep{karimireddy2020scaffold}                
& 0.048{\tiny$\pm$.00} & 0.073{\tiny$\pm$.00}          
& 0.321{\tiny$\pm$.004} & 0.500{\tiny$\pm$.007}             
& 0.492{\tiny$\pm$.006} & 0.550{\tiny$\pm$.010} \\

FedProx \citep{li2020federated}                  
& 0.051{\tiny$\pm$.00} & 0.093{\tiny$\pm$.002}          
& 0.351{\tiny$\pm$.005} & 0.597{\tiny$\pm$.010}             
& 0.509{\tiny$\pm$.005} & 0.553{\tiny$\pm$.008} \\

\midrule
FedSteer                 
& 0.551{\tiny$\pm$.028} & 0.740{\tiny$\pm$.032}          
& \textbf{0.554{\tiny$\pm$.010}} & \textbf{0.657{\tiny$\pm$.011}}             
& 0.599{\tiny$\pm$.009} & 0.656{\tiny$\pm$.009} \\

FedSteer {\tiny(enforce=5)}     
& \textbf{0.620{\tiny$\pm$.032}} & \textbf{0.740{\tiny$\pm$.029}}         
& 0.539{\tiny$\pm$.008} & 0.651{\tiny$\pm$.009}            
& \textbf{0.622{\tiny$\pm$.012}} & \textbf{0.660{\tiny$\pm$.015}} \\

\bottomrule
\end{tabular}
\end{table*}

\subsection{Experimental Setup}
\textbf{FL system:}
We simulate an FL system with $N=100$ clients under independent \textbf{system} and \textbf{data heterogeneity}. For system heterogeneity, we create a weak group (50 clients with participation probability $p_i=0.04$) and a powerful group (50 clients with $p_i=0.16$), resulting in a network-wide average participation rate of $0.1$. For data heterogeneity, a $\gamma$ fraction of clients (the common-label group) are assigned data from the first half of the $L$ total labels, while the remaining $1-\gamma$ fraction (rare-label group) receive data from the second half. To ensure a uniform global label distribution while maintaining an average of 100 datapoints per client, we allocate $\lfloor50/\gamma\rfloor$ datapoints to each common client and $\lfloor50/(1-\gamma)\rfloor$ to each rare client. We conduct experiments under moderate ($\gamma=0.7$) and extreme ($\gamma=0.9$) heterogeneity settings.

\textbf{Datasets \& Models:}
Fashion-MNIST, EMNIST, and CIFAR-10 datasets are used \citep{xiao2017fashion,cohen2017emnist,krizhevsky2009learning}. The model for Fashion-MNIST features two convolutional, two max-pooling, and two fully-connected layers, and the EMNIST model uses a deeper architecture with three of each layer type. 
We apply a PreAct ResNet-18 \citep{he2016resnet} for the CIFAR-10 task. 
Experiments are conducted on a GeForce 1080Ti with 8 different random seeds. We report the average accuracy and variance across these runs.
We provide more details of experimental settings in the Supplementary Material \ref{app:exp}. 

\textbf{Baselines:}
We compare FedSteer against a suite of baseline methods that employ diverse strategies for handling client stale updates.
1) \textit{FedAvg} serves as the fundamental baseline; it aggregates updates only from active clients in each round and discards any stale information.
2) \textit{FedVARP} and \textit{FedStale} both reuse stale updates using a SAGA-like aggregation rule for variance reduction. FedStale further introduces a down-weighting mechanism ($\beta$) to mitigate the negative impact of excessively old gradients. We report FedStale's best result from the selection of $\beta\in\{0.5,0.6,0.7,0.8,0.9\}$. 
3) \textit{SCAFFOLD} corrects for client drift using control variates, which implicitly leverages historical information.
4) \textit{MIFA} takes a direct approach by substituting updates from inactive clients with their most recent stale gradients.
5) \textit{FedProx ($\mu=0.1$)} applies the regularization to limit the local drift caused by client data heterogeneity. 
6)  We also benchmark against a \textit{Full Participation} setting, where the participation rate is $1.0$, to establish a performance upper bound for the given data distribution.

\begin{figure}[t] % Using [ht] for placement preference
    \centering
    % --- Subfigure (a) ---
        \includegraphics[width=0.95\linewidth]{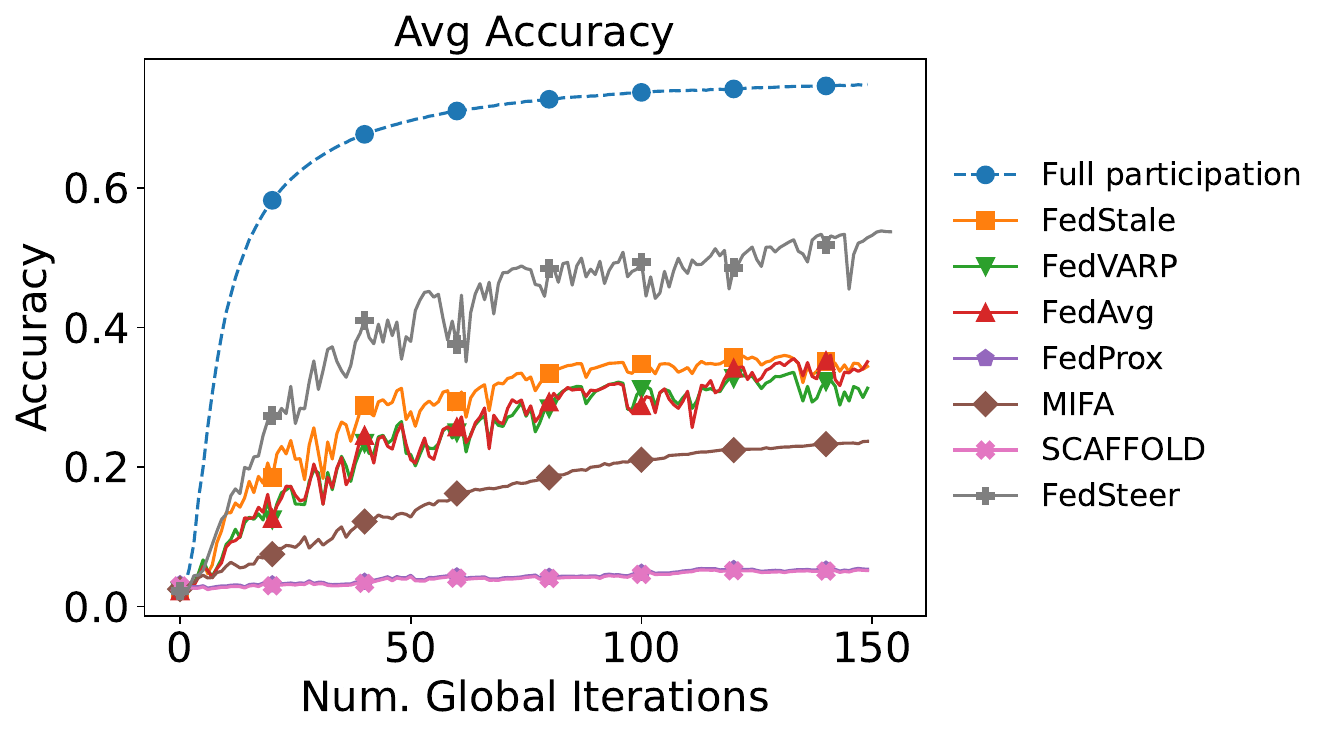}
        
    \caption{Comparison of FedSteer with other baselines ($\gamma=0.9$). FedSteer (blue and orange lines) significantly outperforms all baselines.
    On EMNIST, FedSteer converges stably while baselines struggle, with SCAFFOLD and FedProx's performance collapsing entirely.
    Additional plots of Fashion-MNIST and CIFAR-10 are shown in Appendix \ref{app:plots}.
    }\label{fig:emnist_results}
    \label{fig:p1}
\vspace{-0.1 in}
\end{figure}

% \begin{figure}[t] % Using [ht] for placement preference
%     \centering
    
%     % --- Subfigure (a) ---
%     \begin{subfigure}{\linewidth}
%         \centering
%         \includegraphics[width=0.9\linewidth]{figures/global_avg_acc.pdf}
%         \caption{Experiment results for EMNIST.}
%         \label{fig:emnist_results}
%     \end{subfigure}
    
%     % Add some vertical space between the figures
%     \vspace{0.5cm} 
    
%     % --- Subfigure (b) ---
%     \begin{subfigure}{\linewidth}
%         \centering
%         \includegraphics[width=0.9\linewidth]{figures/global_avg_accF.pdf}
%         \caption{Experiment results for Fashion-MNIST.}
%         \label{fig:fashion_mnist_results}
%     \end{subfigure}
    
%     % --- Overall Caption and Label for the whole figure ---
%     \caption{Comparison of FedSteer with other baselines ($\gamma=0.9$). FedSteer (blue and orange lines) significantly outperforms all baselines.
%     On EMNIST in (a), FedSteer converges stably while baselines struggle, with SCAFFOLD's performance collapsing entirely.
%     On Fashion-MNIST in (b), FedSteer provides 8.8\% higher accuracy compared to the next-best method. [adjust plots]
%     }
%     \label{fig:p1}
% % \vspace{-0.1 in}
% \end{figure}

\subsection{Comparison with baseline methods}
The results in Fig.~\ref{fig:p1} and Table \ref{tab:result} are reported using a core set size of $|\mathcal{X}|=10$ for EMNIST and CIFAR-10, and $|\mathcal{X}|=40$ for Fashion-MNIST, with a regularization factor of $\lambda=0.5$. 
The core set $\mathcal{X}$ was optimized for 5 selection cycles ($T_0=5$). 
We also compare two schemes for updating $\mathbf{s}_i$: the default method (update only when active) and an enforced update every 5 rounds (\texttt{enforce=5}).

Our experiments reveal the vulnerability of existing baselines to highly stale updates under extreme heterogeneity; FedSteer, however, successfully overcomes this issue to deliver a significant performance advantage.
As we mentioned, the experiment setting features a rare-label client group ($1-\gamma$ fraction) that provides critical and unique information. However, half of these clients have a very low participation probability (weak group, $p_i=0.04$), which introduces severely stale updates into the training process. 
As shown in Fig. \ref{fig:p1} and quantified in Table \ref{tab:result}, existing methods are highly susceptible to these highly stale updates, especially under the extreme heterogeneity setting ($\gamma=0.9$) on EMNIST. 
For FedVARP, FedStale, and MIFA, the stale gradients fail to align with the current optimization objective, resulting in final accuracies (0.281, 0.309, and 0.243, respectively) that are on par with or worse than a standard FedAvg baseline (0.314). 
The failure of SCAFFOLD and FedProx is pronounced: their drift correction mechanisms are fundamentally compromised by extreme staleness, leading to a collapse of the training process and a final accuracy of just 0.048 and 0.051.

FedSteer, however, achieves stable convergence by avoiding the direct use of stale updates.
Instead, it leverages their projection coordinates on a dynamic subspace $Q_t$, which is continuously steered by more active clients. On EMNIST ($\gamma=0.9$), the enforced-update version of FedSteer achieves a final accuracy of 0.620, nearly doubling the performance of the best-performing baseline. On Fashion-MNIST ($\gamma=0.9$), FedSteer provides an 8.8\% higher accuracy (0.554) compared to the next-best method, FedAvg ($0.509$). Under moderate heterogeneity ($\gamma=0.7$), where baselines are more competitive, FedSteer consistently exceeds the top-performing methods by 7.8\%, approaching the ideal accuracy of full participation. 
FedSteer maintains this advantage on the larger and more complex CIFAR-10 dataset, achieving accuracies of $0.622$ and $0.660$ under $\gamma=0.9$ and $\gamma=0.7$, respectively.

\subsection{Impact of Core Set Size and Regularization}

\begin{figure}
    \centering
    \includegraphics[width=0.7\linewidth]{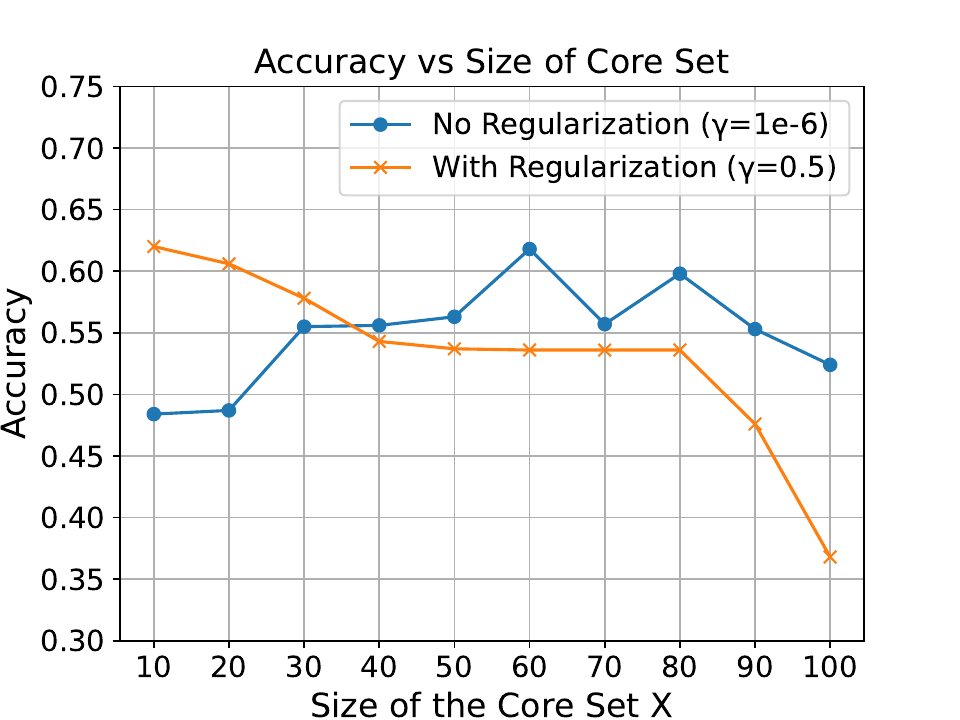}
    \caption{Final test accuracy on EMNIST versus core set size. The plot compares a strongly regularized ($\lambda=0.5$) and a virtually unregularized ($\lambda=1\mathrm{e}{-6}$) FedSteer algorithm. 
    Regularization proves crucial for small core sets to prevent overfitting. 
    For large sets, performance degrades because the over-specialized coordinates can be sensitive to the evolving subspace when reused in later rounds. 
    }
    \label{fig:Xsize}
    \vspace{-0.1in}
\end{figure}

We evaluate the impact of the core set size $|\mathcal{X}|$ and regularization on FedSteer's performance on EMNIST, with results presented in Fig.~\ref{fig:Xsize}. In this experiment, projection coordinates are recomputed every five rounds, which requires them to be robust enough to remain effective as the subspace evolves. The results reveal a critical trade-off between the representational capacity of the subspace and the coordinates' generalization ability.
\textbf{For a small core sets ($|\mathcal{X}| \le 30$)}, the subspace provides a limited representation of the gradient landscape. Strong regularization is crucial here, as it prevents the coordinates from overfitting to the small basis, producing more robust estimates and leading to higher accuracy.
\textbf{For medium core sets ($40 \le |\mathcal{X}| \le 80$)}, as the core set grows, the subspace's capacity to represent client gradients increases. In this range, the unregularized approach starts to outperform the regularized one by leveraging the richer, more representative basis without the constraint of a penalty.
\textbf{For large core sets ($|\mathcal{X}| > 80$)}, the performance of both methods degrades. This decline can be attributed to the coordinates becoming over-specialized. When the basis becomes too large, the subspace's high dimensionality allows it to fit client gradients too closely at the moment of computation. These over-specialized coordinates are highly sensitive to the subspace's evolution across subsequent rounds; when they are reused, their effectiveness diminishes, leading to inaccurate gradient reconstructions and a drop in model accuracy.

In Appendix \ref{app:exp}, we provide the detailed experimental setup (\ref{app:setup}), complexity analysis against standard baselines (\ref{app:communication}), and additional results including a subsampled core set selection strategy for large-scale deployment with reduced complexity (\ref{app:subsampling}), convergence plots on Fashion-MNIST and CIFAR-10 (\ref{app:plots}), and the stability of coordinate reuse under evolving subspaces (\ref{app:reusingsi}).

\section{Conclusion}
\label{sec:conclusion}
In this work, we introduce FedSteer, a novel algorithm that creates a dynamic, low-dimensional gradient subspace from a small client core set. It projects a client's gradient onto this subspace, computing and caching stable, low-dimensional coordinates for reuse. For inactive clients, these coordinates ``steer'' outdated information toward the current global objective by being applied to the newly evolved subspace. Complemented by a memory-efficient selective caching strategy, this mechanism is proven to minimize global update variance and achieve a tighter convergence bound. Experiments confirm FedSteer significantly outperforms baselines, prevents training collapse under extreme heterogeneity, and delivers significant accuracy gains.

\section{Acknowledgements}
This work was supported by NSF CNS-2106891, TREES: ANR-24-TSIA-0004, and A*STAR under its IAF-ICP programme (H25-MCP3438).